\documentclass{article}

 \usepackage[preprint]{neurips_2026}

\usepackage[utf8]{inputenc} % allow utf-8 input
\usepackage[T1]{fontenc}    % use 8-bit T1 fonts
\usepackage{hyperref}       % hyperlinks
\usepackage{url}            % simple URL typesetting
\usepackage{booktabs}       % professional-quality tables
\usepackage{amsfonts}       % blackboard math symbols
\usepackage{nicefrac}       % compact symbols for 1/2, etc.
\usepackage{microtype}      % microtypography
\usepackage{xcolor}         % colors

\usepackage{microtype}
\usepackage{graphicx}
\usepackage{subcaption}
\usepackage{booktabs} % for professional tables

\usepackage{float}
\usepackage[normalem]{ulem}
\usepackage{lineno}
\usepackage{longtable}
\usepackage{wrapfig}

\usepackage{alltt}
\usepackage{fvextra}

\DefineVerbatimEnvironment{MyVerbatim}{Verbatim}{breaklines=true}

\usepackage{amsmath}
\usepackage{amssymb}
\usepackage{mathtools}
\usepackage{amsthm}
\usepackage{enumitem}

\theoremstyle{plain}
\newtheorem{theorem}{Theorem}[section]

\newtheorem{lemma}[theorem]{Lemma}

\theoremstyle{definition}

\theoremstyle{remark}

\usepackage[textsize=tiny]{todonotes}

\title{Adaptive Margin RLHF via Preference over Preferences}

\author{
  Yaswanth Chittepu$^{1*}$, Prasann Singhal$^{2}$, Greg Durrett$^{2}$, Scott Niekum$^{1 \dag}$ \\
  $^{1}$University of Massachusetts Amherst, 
  $^{2}$University of Texas at Austin \\
  \thanks{Correspondence to: \texttt{ychittepu@umass.edu}}
}

\begin{document}

\maketitle

\begin{abstract}
  Margin-based optimization is fundamental to improving generalization and robustness in classification tasks. In the context of reward model learning from preferences within Reinforcement Learning from Human Feedback (RLHF), existing methods typically rely on no margins, fixed margins, or margins that are simplistic functions of preference ratings. However, such formulations often fail to account for the varying strengths of different preferences or they rely on noisy margin information derived from preference ratings. Furthermore, many existing methods that use adaptive margins assume access to accurate preference scores, which can be difficult for humans to provide reliably. We propose leveraging preferences over preferences, that is, annotations indicating which of two preferences reflects a stronger distinction, to infer adaptive margins on a per-datapoint basis. Such preference-over-preference annotations are general and can be incorporated into both standard RLHF reward modeling objectives and direct alignment losses. As a concrete instantiation, we introduce DPO-PoP, an extension to Direct Preference Optimization (DPO) that incorporates adaptive margins from preference-over-preference supervision, enabling improved discriminative and generative performance.
  Additionally, we show a tradeoff between discriminative and generative performance and propose two sampling strategies for gathering preference-over-preference labels to navigate it.
  % These results suggest that integrating preference-over-preference information, which requires less precision to be provided accurately, can improve discriminative and generative performance without adding significant complexity. 
  % Additionally, we show that there is a tradeoff between discriminative and generative performance: improving test classification accuracy, particularly by correctly labeling weaker preferences at the expense of stronger ones, can lead to a decline in generative quality. To navigate this tradeoff, we propose two sampling strategies to gather preference-over-preference labels: one favoring discriminative performance and one favoring generative performance.
\end{abstract}

\section{Introduction}
\label{sec: introduction}
% Margin based approaches have been pivotal to the design and understanding of Classification algorithms. Margin-defined-as the distance between a datapoint and the decision boundary plays a crucial role in understanding model generalization \textcolor{red}{(cite SVMs and other papers}. In classical machine learning, support vector machines (SVMs) explicitly maximize the minimum margin between classes, a strategy that has been shown to improve generalization performance (Cortes Vapnik, 1995). Ensemble methods like AdaBoost were later shown to implicitly increase the margin distribution of training examples, with theoretical work suggesting that larger margins—especially across the entire distribution—correlate with lower generalization error (Schapire et al., 1998; Koltchinskii Panchenko, 2002).

% Herbirch Weston proposed Adaptive Margin SVMs for Classification, which can use different margins values for different training datapoints. They provide bounds on the generalization error of AM-SVMs and show that AM-SVMs are robust to outliers and that Adaptive Margins control model complexity. \textcolor{red}{Deep Ranking with Adaptive Margin Triplet Loss}

Margin-based approaches have been pivotal in the design and analysis of classification algorithms. In classical machine learning, the margin, defined as the distance between a decision boundary and data points, acts as a proxy for confidence and plays a critical role in improving generalization. For example, Support Vector Machines (SVMs) maximize the minimum margin, which has been shown to enhance robustness and reduce overfitting \citep{cortes1995support}. Ensemble methods like AdaBoost \citep{freund1996experiments} also leverage margin-based generalization, as boosting algorithms implicitly seek to increase the margin distribution across training samples \citep{schapire1998boosting}.

%\sn{not sure what you mean by informativeness and what the connection to margin is}%
%\yc{Although fixed-margin strategies have proven effective, they assume that all training examples are equally difficult and equally informative—i.e., they contribute similar learning signal to the model. This assumption is often violated in real-world data, where some comparisons are more ambiguous or less reliable than others.} \sn{I get what you're trying to say, but one issue: sometimes when two outputs are close together the pref shouldn't be trusted because there is more opportunity for error; on the other hand, there's a paper I've seen that shows that close together examples are the most informative, because the model generally already knows how to classify the really far apart ones.}.%
Although fixed-margin strategies have proven effective, they assume fixed and equal margin for all training data points. This has motivated the development of adaptive margin approaches, where the margin varies across examples based on criteria such as sample difficulty, uncertainty, or class imbalance. Adaptive Margin SVMs \citep{westonASV} use different margin values for different training data points and provide bounds on the generalization error, justifying its robustness against outliers. Furthermore, methods such as CurricularFace \citep{huang2020curricularface}, AdaCos \citep{zhang2019adacosadaptivelyscalingcosine}, and adaptive triplet losses \citep{ha2021deep} have shown that adapting the margin dynamically during training leads to more stable optimization and better generalization.
% , particularly in settings such as face recognition or imbalanced classification.

In Reinforcement Learning from Human Feedback (RLHF), pairwise preference data from humans is used to learn a reward function or policy. The Bradley-Terry (BT) model \citep{bradley1952rank} is widely used to model pairwise preference data, where the probability of preferring one output over another is determined by the difference in their reward scores.
%\sn{why do you say latent here? Confusing.} \yc{latent = unobserved reward function, from Graphical models}reward scores \sn{there's too much baggage around "latent" from NNs. Just say reward function/scores by itself. You're already saying it is being inferred.}%. 
This preference model is commonly used in the alignment of large language models (LLMs) \citep{ouyang2022training, touvron2023llama}, in which a reward function is learned to rank outputs based on human preferences, and subsequently used to optimize the policy.

% In reward modeling, especially in the context of learning from human preferences, margin-based reasoning is also central. The Bradley-Terry (BT) model \cite{bradley1952rank} is widely used to model pairwise preference data, where the probability of preferring one output over another is determined by the difference in their latent reward scores. This model underpins many modern reward modeling pipelines used to align large language models (LLMs), where a reward function is learned to rank outputs from best to worst based on human preferences \cite{ouyang2022training, rafailov2023direct}.

Current reward modeling approaches generally fall into two categories. Some methods treat all preferences equally by applying no margin at all \citep{ouyang2022training}. Others incorporate unequal treatment by introducing adaptive margins, which are typically derived in one of two ways: either from scalar scores assigned to preferences by human annotators or language models \citep{touvron2023llama, wang2025helpsteer2preferencecomplementingratingspreferences}, or from the outputs of learned reward models \citep{wang2024secretsrlhflargelanguage, qin2024understandinginfluencerewardmargin, amini2024direct, wang2024rewarddifferenceoptimizationsample}.
%However, current reward modeling approaches typically treat all comparisons equally, using either no margins \citep{ouyang2022training}, %
%fixed margins \sn{add citation for fixed margins} \yc{Havent seen any works that use fixed margins in RM, the closest is SLiC, but that is a Direct Alignment ALgorithm, not RM}%
%or unequally and use margins based on scores assigned to preferences by LLM or human annotators.%
%\sn{I think you mean that the outputs are scores, not the preference}\yc{margin based on scores of preferences?} obtained from human/LLM feedback 
%\citep{touvron2023llama, wang2025helpsteer2preferencecomplementingratingspreferences} or reward models \citep{wang2024secretsrlhflargelanguage, qin2024understandinginfluencerewardmargin, amini2024direct}%. 
%Using constant or no margin information fails to account for the varying difficulty or ambiguity of different preferences \sn{what is the "difficulty" of a preference? I think you mean magnitude or strength}.%
Using constant or no margin information fails to account for the varying strength of different preferences. Obtaining preference strength information from preference scores allows us to use adaptive margin information, but requires us to collect scalar feedback from LLMs or humans.
%\sn{don't capitalize random words like preference or adaptive margin or humans}.% 
% However, judgments for different tasks differ among humans, and may be mapped to a labeling scheme, such as Likert ratings, in different ways \citep{wadhwa2024usingnaturallanguageexplanations}. \sn{you're making it sound like likert ratings are a form of preference, but the are applied to one output at a time and a preference can be derived from this. This needs to be clear} \sn{be more specific. Say something about humans not necessarily being calibrated on scores, even if they would agree overall on which option is better.}.

\begin{wrapfigure}{r}{0.6\columnwidth}
    \centering
    \includegraphics[width=0.58\columnwidth]{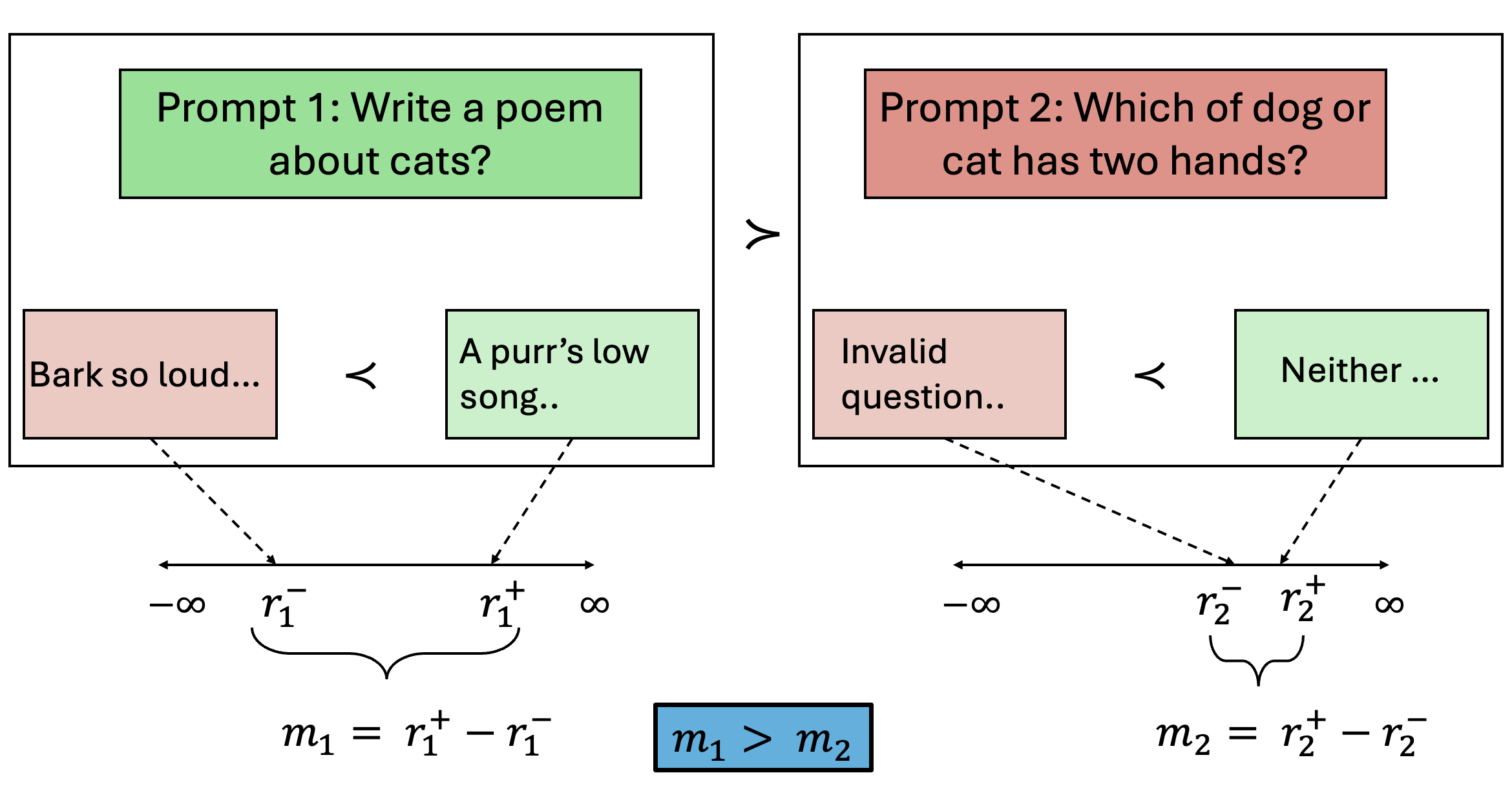}
    \caption{\small A pictorial illustration of the PoP framework. A preference is stronger than another when the reward difference between its preferred and dispreferred responses is larger. The reward difference of the weaker preference serves as the margin for the stronger.}
    \label{fig:pop-illustrated}
\end{wrapfigure}

% \begin{figure}[t]
%     \centering
% \includegraphics[width=\columnwidth]{neurips2026/figs/pop_illustrated_5.png}
%     \caption{A pictorial illustration of the PoP framework. A preference is stronger than another when the reward difference between its preferred and dispreferred responses is larger. The reward difference of the weaker preference in the pair serves as the margin for the stronger preference.}
%     \label{fig:pop-illustrated}
% \end{figure}

Specifying preference strength typically requires assigning a numerical score, which is poorly calibrated in practice. Labeling schemes such as Likert ratings, where annotators rate responses individually rather than comparatively, may not be consistently calibrated---even annotators who agree on which response is better may assign inconsistent scores \citep{wadhwa2024usingnaturallanguageexplanations}. Comparative annotation approaches such as Best-to-Worst Scaling (BWS) sidestep this by grounding judgments in relative comparisons, and have been shown to produce significantly more reliable annotations as a result \citep{kiritchenko2017best, burton2019best}.

We extend this logic one level further. Rather than asking annotators to place a preference on an absolute scale, we ask them to compare two preferences and indicate which reflects the stronger preference. This preference-over-preference (PoP) formulation is cognitively natural: it reduces to a single binary judgment about relative magnitudes. Critically, the PoP formulation is grounded in MACBETH \citep{banae1994macbeth}, a decision-analysis method that elicits cardinal value scales from qualitative pairwise comparisons of differences in attractiveness. PoP brings this validated principle into the preference learning setting, with the added benefit that these binary comparisons are informationally richer than discrete ratings: we can infer continuous real-valued margins rather than being confined to the coarse, poorly-anchored discrete options of rating scales.

In this work, we propose a framework for adaptive margin RLHF via Preference-over-Preference (PoP) supervision, illustrated in Figure~\ref{fig:pop-illustrated}. Rather than deriving margins from noisy and improperly calibrated scalar annotations, PoP infers them from comparative judgments between preferences, which are more reliably calibrated across annotators. The PoP supervision framework is general: it can be incorporated into the reward modeling objective in standard RLHF pipelines, or into direct alignment losses directly. As a concrete instantiation, we propose DPO-PoP, which integrates PoP supervision into the DPO framework \citep{rafailov2024directpreferenceoptimizationlanguage}, enabling margin-aware alignment using only supervised learning.
We demonstrate that collecting PoP supervision is a simple and effective way to improve LLM performance, with DPO-PoP variants improving over all baselines in both discriminative and generative metrics. Moreover, we highlight a tradeoff between discriminative performance, as measured by test classification accuracy, and generative performance, as measured by win rate, where improving classification accuracy on weaker preferences at the expense of stronger ones can lead to a decline in generative quality. To navigate this tradeoff, we propose two sampling strategies for generating preference-over-preference labels: iterative sampling, which favors discriminative performance, and random sampling, which favors generative performance.

% In this work, we propose a novel Adaptive Margin DPO loss that integrates PoP supervision into the Direct Preference Optimization (DPO) framework \cite{rafailov2023direct}, directly aligning a LLM from preference data using only Supervised Learning. We use the Preference over Preference information for each datapoint to infer an adaptive margin; one for each datapoint. We show that gathering Preference over Preferences is a simple and effective approach to improve the performance of LLMs and their alignment with human intent

% Our method treats each preference-over-preference instance as a higher-order comparison and introduces a learnable or data-driven margin to reflect the strength or clarity of each comparison. By combining the structure of the BT model with adaptive margins, we directly align a language model from preference data without needing scalar reward models or reinforcement learning. This approach leads to more sample-efficient, robust, and interpretable alignment, especially when dealing with nuanced or ambiguous preferences.

\section{Background}
\label{sec: background}
\subsection{Reward Modeling} In the reward modeling stage of RLHF, a reward model is trained to assign scalar scores to prompt-response pairs, indicating how well a response aligns with human preferences. This process relies on a preference dataset $\mathcal{D}_{\text{pref}} = {(x_i, y^+_i, y^-_i)}_{i=1}^{N}$, where $x_i$ is a prompt, $y^+_i$ is the preferred response, and $y^-_i$ is the dispreferred response. The Bradley-Terry (BT) model \citep{bradley1952rank} is commonly used to model preference likelihoods.

% \begin{equation}
%     P(y^{+} \succ y^{-}) = \frac{e^{r(x,y^{+})}}{e^{r(x,y^{+})}+e^{r(x,y^{-})}} = \sigma(r(x,y^{+})-r(x,y^{-}))
% \end{equation}

\begin{equation}
\begin{aligned}
P(y^{+} \succ y^{-})
% &= \frac{e^{r(x,y^{+})}}{e^{r(x,y^{+})}+e^{r(x,y^{-})}} \\
&= \sigma\!\left(r(x,y^{+})-r(x,y^{-})\right)
\end{aligned}
\end{equation}

Here, $r$ denotes the 
%\sn{again, I think you should get rid of latent and just say reward}
reward assigned to a prompt-response pair, and $\sigma$ denotes the sigmoid function. We parameterize the reward function as $r_\phi$, and use it to approximate the ground-truth reward function by maximizing the likelihood of the observed preference data under the Bradley-Terry model (Equation~\ref{eq:mle-bt-rm}). For more details on the RLHF pipeline, refer to Appendix $\ref{ap:rlhf}$

\begin{equation}
\label{eq:mle-bt-rm}
    \min_{\phi} -\mathbb{E}_{(x,y^{+},y^{-}) \sim \mathcal{D}_\text{pref}}[\log \sigma(r_{\phi}(x,y^{+})-r_{\phi}(x,y^{-}))]
\end{equation}

\subsection{Direct Preference Optimization}

Direct Preference Optimization (DPO) \citep{rafailov2024directpreferenceoptimizationlanguage} belongs to a class of algorithms, called Direct Alignment Algorithms (DAAs) \citep{rafailov2024scalinglawsrewardmodel}, which aim to directly align a policy from preference data via supervised learning, without having to learn a reward model or use reinforcement learning.
%\sn{I've been going through and un-capitalizing reinforcement learning wherever I see it -- it doesn't require capitalization}
DPO utilizes the closed form solution of the optimal KL regularized reward policy \citep{peters2007reinforcement,peng2019advantageweightedregressionsimplescalable} and expresses the rewards in the Bradley-Terry preference model \citep{bradley1952rank}, directly in terms of the optimal policy. This allows us to learn a parameterized optimal policy directly from the preference data, using Equation $\ref{eq:dpo}$

% \begin{equation}
% \label{eq:dpo}
%     \mathcal{L}_{DPO}(\pi_{\theta}; \pi_{ref}) = \mathbb{E}_{(x,y^{+},y^{-}) \sim \mathcal{D}_\text{pref}}\left[-\log \sigma\left(\beta \log \frac{\pi_{\theta}(y^{+}\vert x)}{\pi_\text{ref}(y^{+}\vert x)}-\beta \log \frac{\pi_{\theta}(y^{-}\vert x)}{\pi_\text{ref}(y^{-}\vert x)}\right)\right]
% \end{equation}

\begin{equation}
\label{eq:dpo}
\begin{aligned}
\mathcal{L}_{\text{DPO}}(\pi_{\theta}; \pi_{\text{ref}})
&= \mathbb{E}_{(x,y^{+},y^{-}) \sim \mathcal{D}_{\text{pref}}}
\Bigl[
-\log \sigma\Bigl(
\beta \log \frac{\pi_{\theta}(y^{+}\vert x)}{\pi_{\text{ref}}(y^{+}\vert x)} 
- \beta \log \frac{\pi_{\theta}(y^{-}\vert x)}{\pi_{\text{ref}}(y^{-}\vert x)}
\Bigr)
\Bigr]
\end{aligned}
\end{equation}

The implicit reward assigned by the DPO model to a response $y$ given a prompt $x$ is $\beta \log \frac{\pi_{\theta}(y\vert x)}{\pi_\text{ref}(y\vert x)}$.

\subsection{Margins in Reward Modeling}

Margins can be incorporated into the reward modeling phase of the RLHF pipeline to enforce not only that the reward model ranks the preferred response higher than the dispreferred one, but also that it assigns a sufficiently large difference in reward scores—either through fixed or adaptive margins. The margin-based reward modeling loss, for a margin $m(x,y^{+},y^{-})$, can be expressed as:

% Margins can be employed the in Reward Modeling phase of the RLHF pipeline, to ensure that the reward model not only ranks the preferred response higher than the dispreferred response, and also assigns more discrepant scores to responses that are more different in the case of adaptive margins, or esnures atleast a fixed gap between reward scores for chosen and rejected responses. Margins can be either fixed or adaptive. Fixed Margins ensure that the difference between the scores assigned to the preferred and dispreffered response is greater than or equal to the margin. This leads to better generalization performance, similar to Support Vector Machines. Adaptive Margins ensure that stronger preferences have larger discrepancies between their reward scores than weaker preferences, alongside ensuring the ordering of the preference. The Margin based Reward modeling loss can be expressed as 

% \begin{equation}
%     \min_{\phi} -\mathbb{E}_{(x,y^{+},y^{-}) \sim \mathcal{D}_\text{pref}}[\log \sigma(r_{\phi}(x,y^{+})-r_{\phi}(x,y^{-}) - m(x,y^{+},y^{-})]
% \end{equation}

\begin{equation}
\begin{aligned}
\min_{\phi} \quad -\mathbb{E}_{(x,y^{+},y^{-}) \sim \mathcal{D}_{\text{pref}}}
\Bigl[
&\log \sigma\Bigl(
r_{\phi}(x,y^{+}) - r_{\phi}(x,y^{-})- m(x,y^{+},y^{-})
\Bigr)
\Bigr]
\end{aligned}
\end{equation}

% Here $m(x,y^{+},y^{-})$ denotes the margin term. In the fixed margin setting this can be a constant. In the adaptive-margin setting, it can be defined as a function of the preference instance, for example, based on the degree of discrepancy between the preferred and dispreferred responses.

\section{Method: Adaptive Margin DPO with Preferences over Preferences}

\label{sec: method}
To obtain adaptive margin information, in which each preference datapoint is assigned a different margin, and stronger preferences are associated with larger margins than weaker ones, we propose preferences over preferences (PoP) supervision. Given two standard preference comparisons, such as $A \succ B$ and $C \succ D$, we collect a label indicating which of the two preferences is stronger, from a labeler.
% \sn{here and elsewhere, you shouldn't restrict the writing to human labels, because it can apply to RLAIF as well. You can just say "from a labeler". Especially because the Ultrafeedback data later isn't even from humans}.
For example, if the supervision indicates that $(A \succ B) \succ (C \succ D)$, this means that the discrepancy between $A$ and $B$ is greater than that between $C$ and $D$ under the ground-truth reward function $r$. Formally, this implies:

\[
r(A) - r(B) > r(C) - r(D)
\]

% This insight allows us to use the current margin from the weaker preference (e.g., $r(C) - r(D)$) as a lower bound or adaptive margin for the stronger preference (e.g., $A \succ B$) \sn{the language isn't totally clear here that you are not trying to regress the margin to a specific value, but are saying that it has to be at least that large}. Intuitively, the stronger a preference is judged to be, the larger the margin we should enforce between its preferred and dispreferred responses.

This insight allows us to treat the margin from the weaker preference (e.g., $r(C) - r(D)$) as a lower bound on the margin for the stronger preference (e.g., $A \succ B$). Rather than regressing to a specific value, we enforce that the margin for the stronger preference must be at least as large as that of the weaker one. 
% Intuitively, the stronger a preference is judged to be, the larger the margin we should enforce between its preferred and dispreferred responses.

We assume access to a dataset of preference over preference examples:

\[
\mathcal{D}_{\text{PoP}} = \left\{ \left( (x_{s_i}, y_{s_i}^{+}, y_{s_i}^{-}),\ (x_{w_i}, y_{w_i}^{+}, y_{w_i}^{-}) \right) \right\}_{i=1}^N
\]

Here, $(x_{s_i}, y_{s_i}^{+}, y_{s_i}^{-})$ represents the stronger preference in the pair, where $x_{s_i}$ is the prompt, $y_{s_i}^{+}$ is the preferred response, and $y_{s_i}^{-}$ is the dispreferred response. Similarly, $(x_{w_i}, y_{w_i}^{+}, y_{w_i}^{-})$ denotes the weaker preference, where $x_{w_i}$ is the prompt, $y_{w_i}^{+}$ is the preferred response, and $y_{w_i}^{-}$ is the dispreferred response. Note that, unlike in standard reward modeling datasets, the prompts $x_{s_i}$ and $x_{w_i}$ can differ within a single PoP example, as PoP supervision compares the strength of entire preference instances, not individual responses.

% To get Adaptive Margin information, wherein each data point can have different margins, with stronger preferences having larger margins than weaker preferences, we leverage Preferences over Preferences. Given two preferences $A \succ B$, and $C \succ D$. We obtain a label on which preference is stronger, either from a human or LLM. Let's assume this gives us the Preference over preference $(A \succ B) \succ (C \succ D)$. The underlying idea is that we can use this ranking over preferences, to learn an adaptive margin for each data point. Assuming, the ground truth latent reward function is denoted by $r$, we have that $r(A)-r(B) > r(C)-r(D)$. We can use $r(C)-r(D)$ as the margin for the Preference data point $A \succ B$. 

% Given a dataset of Preferences over Preferences $\mathcal{D}_{pop} = \{(x^{+}_{i}, y^{++}_{i}, y^{+-}_{i}),(x^{-}_{i}, y^{-+}_{i}, y^{--}_{i})\}_{i=1}^{N}$. Here $x^{+}$ denotes the prompt for the stronger preference in the pair, $y^{++}$ denotes the preferred response in the stronger preference, $y^{+-}$ denotes the dispreferred response in the stronger preference. Similarly, $x^{-}$ denotes the prompt for the weaker preference in the pair, $y^{-+}$ denotes the preferred response in the weaker preference, $y^{--}$ denotes the dispreferred response in the weaker preference. Unlike the Preference dataset used in Standard Reward Modeling, the two prompts $x^{+}$, $x^{-}$ can be different in a single preference pair.

We can express the adaptive margin reward modeling objective on a dataset of preferences over preferences as follows

\begin{align}
\label{eq:rm-pop}
\min_{\phi} \ \mathbb{E}_{\mathcal{D}_{\text{PoP}}
} \Big[
- \log \sigma \big(& r_{\phi}(x_{s}, y_{s}^{+}) - r_{\phi}(x_{s}, y_{s}^{-}) - \text{sg} \left[ r_{\phi}(x_{w}, y_{w}^{+}) - r_{\phi}(x_{w}, y_{w}^{-}) \right] \big)
\Big]
\end{align}

% \begin{equation}{\label{eq:rm-pop}}
%     \min_{\phi} \mathbb{E}_{(x^{+}, y^{++}, y^{+-}),(x^{-}, y^{-+}, y^{--}) \sim \mathcal{D}_{pop}}\left[-\log \sigma \left(r_{\phi}(x^{+}, y^{++})-r_{\phi}(x^{+}, y^{+-}) -sg[r_{\phi}(x^{-}, y^{-+})-r_{\phi}(x^{-}, y^{--})]\right)\right]
% \end{equation}

Here, \( \text{sg}[\cdot] \) denotes the stop-gradient operator. Although the adaptive margin is computed using the reward model \( r_\phi \), we treat the margin derived from the weaker preference as a \textit{fixed reference} during optimization. Applying the stop-gradient operator ensures that gradients do not propagate through this margin term, thereby preventing it from influencing updates to the reward model parameters \( \phi \). Without the stop-gradient operator, the objective would incentivize parameters that invert the weaker preference to minimize the loss.

We use the closed-form solution for the optimal policy of a KL regularized reward problem to express the rewards directly in terms of the optimal policy, as in DPO \citep{rafailov2024directpreferenceoptimizationlanguage}. Parameterizing the optimal policy by $\theta$, we end up with the DPO Preference-over-Preference loss. Denoting the implicit margin of a preference pair as $m_{\theta}(x,y^+,y-) = \beta \left(
            \log \frac{\pi_{\theta}(y^{+} \mid x)}{\pi_\text{ref}(y^{+} \mid x)}
            - \log \frac{\pi_{\theta}(y^{-} \mid x)}{\pi_\text{ref}(y^{-} \mid x)}
        \right)$, the DPO Preference-over-Preference loss is:
% From the closed form solution of the optimal policy $\pi^{*}(y \vert x) \propto \pi_{ref}(y \vert x) e^{\frac{r(x,y)}{\beta}}$ \citep{peters2007reinforcement,peng2019advantageweightedregressionsimplescalable}, we can express the reward in terms of the optimal policy as follows

% \begin{equation}\label{eq:rew_to_policy}
%     r(x,y) = \beta \times \left(\log \frac{\pi^{*}(y \vert x)}{\pi_{ref}(y \vert x)} + Z(x)\right)
% \end{equation}

% Where $Z(x)$ is a normalizing factor that depends only on the prompt $x$. 

% \begin{equation}{\label{eq:dpo-pop}}
%     \min_{\theta} \mathbb{E}_{\mathcal{D}_{pop}}\left[-\log \sigma \left(\beta(\log \frac{\pi_{\theta}(y^{++} \vert x^{+})}{\pi_{ref}(y^{++} \vert x^{+})}-\log \frac{\pi_{\theta}(y^{+-} \vert x^{+})}{\pi_{ref}(y^{+-} \vert x^{+})}) -sg[\beta(\log \frac{\pi_{\theta}(y^{-+} \vert x^{-})}{\pi_{ref}(y^{-+} \vert x^{-})}-\log \frac{\pi_{\theta}(y^{--} \vert x^{-})}{\pi_{ref}(y^{--} \vert x^{-})})]\right)\right]
% \end{equation}
{
\setlength{\abovedisplayskip}{2pt}%
 \setlength{\belowdisplayskip}{2pt}%
 \setlength{\abovedisplayshortskip}{2pt}%
 \setlength{\belowdisplayshortskip}{2pt}%
\begin{align}
\label{eq:dpo-pop}
\min_{\theta} \mathbb{E}_{\mathcal{D}_\text{PoP}} \Big[
    -\log \sigma \Big(&
        m_{\theta}(x_s, y_{s}^{+},y_{s}^{-}) - \text{sg} \left[
            m_{\theta}(x_w, y_{w}^{+},y_{w}^{-})
        \right]
    \Big)
\Big]
\end{align}}

% \begin{align}
% \label{eq:dpo-pop}
% \min_{\theta}\;
% \mathbb{E}_{\mathcal{D}_{\text{PoP}}}
% \Bigg[
% -\log \sigma \Bigg(
% &\beta \Bigg(
%     \log \frac{\pi_{\theta}(y_{s}^{+} \mid x_{s})}{\pi_{\text{ref}}(y_{s}^{+} \mid x_{s})}
%     - \log \frac{\pi_{\theta}(y_{s}^{-} \mid x_{s})}{\pi_{\text{ref}}(y_{s}^{-} \mid x_{s})}
% \Bigg) \nonumber \\
% &\;\; - \text{sg} \Bigg(
%     \beta \Bigg(
%         \log \frac{\pi_{\theta}(y_{w}^{+} \mid x_{w})}{\pi_{\text{ref}}(y_{w}^{+} \mid x_{w})}
%         - \log \frac{\pi_{\theta}(y_{w}^{-} \mid x_{w})}{\pi_{\text{ref}}(y_{w}^{-} \mid x_{w})}
%     \Bigg)
% \Bigg)
% \Bigg)
% \Bigg]
% \end{align}

The DPO Preference-over-Preference (DPO-PoP) objective enables margin-aware alignment directly from PoP data using supervised learning. However, Equation~\ref{eq:dpo-pop} suffers from unstable gradients due to unbounded margins, resulting in a rapidly fluctuating loss that can explode during training.
To mitigate this, we clip the margin values to lie within a fixed interval \([0, M_{\text{max}}]\), where \(M_{\text{max}}\) is a user-specified constant. Margin values outside this range are clipped to the nearest endpoint, using a clipping function \(\text{clip}_{[0, M_{\text{max}}]}\), which improves optimization stability. Additionally, to further stabilize training, we compute the margins using a slowly-updated target policy \( \pi_{\hat{\theta}} \), whose parameters $\hat{\theta}$ track the policy \( \pi_{\theta} \) via Polyak averaging over the model parameters \( \theta \). This prevents the margin estimates from changing too rapidly across training steps. 
With these modifications, our final DPO-PoP objective is given by Equation~\ref{eq:dpo-pop-clip}

{\setlength{\abovedisplayskip}{2pt}%
 \setlength{\belowdisplayskip}{2pt}%
 \setlength{\abovedisplayshortskip}{2pt}%
 \setlength{\belowdisplayshortskip}{2pt}%
\begin{align}
\label{eq:dpo-pop-clip}
\min_{\theta} \; \mathbb{E}_{\mathcal{D}_\text{PoP}}\Bigg[
    -\log \sigma & \Bigg(
         m_{\theta}(x_s, y_{s}^{+},y_{s}^{-})- \text{sg}\Big[
            \text{clip}_{[0,M_{max}]} \Big(
            m_{\hat{\theta}}(x_w, y_{w}^{+},y_{w}^{-})
            \Big)
        \Big]
    \Bigg)
\Bigg]
\end{align}
}
% \begin{align}
% \label{eq:dpo-pop-clip}
% \min_{\phi} \ \mathbb{E}_{\mathcal{D}_{\text{PoP}}
% } \Big[
% - \log \sigma \big(& r_{\phi}(x_{s}, y_{s}^{+}) - r_{\phi}(x_{s}, y_{s}^{-}) \nonumber \\
% & - \text{sg} \left[ r_{\phi}(x_{w}, y_{w}^{+}) - r_{\phi}(x_{w}, y_{w}^{-}) \right] \big)
% \Big]
% \end{align}

From Equation~\ref{eq:dpo-pop-clip}, we see that DPO-PoP requires only one additional forward pass to compute margins, when compared to DPO. As a result, the additional computational overhead is minimal and the method is easy to integrate into existing DPO implementations.

\section{Results}

\label{sec: results}
We focus on the following research questions: \textbf{[Q1]} Does using DPO-PoP lead to models with improved discriminative ability? \textbf{[Q2]} Does using DPO-PoP lead to models with improved generative ability? We investigate these questions by evaluating the performance of our models on the test split of the UltraFeedback dataset \citep{cui2024ultrafeedbackboostinglanguagemodels} and external benchmarks such as RewardBench \citep{lambert2024rewardbenchevaluatingrewardmodels} and AlpacaEval-2 \citep{dubois2025lengthcontrolledalpacaevalsimpleway}. More importantly, we also investigate \textbf{[Q3]}: Do the same trends observed in Q1 and Q2 hold when PoP annotations are gathered from an LLM annotator? This is important because it sheds light on whether PoP annotation is a practically viable alternative to rating-scale annotations for improving performance.
\subsection{Synthetic Data Experiments}
\subsubsection{Generating the Preference over Preference data}

% Shorter

We use the UltraFeedback \citep{cui2024ultrafeedbackboostinglanguagemodels} binarized dataset\footnote{\texttt{HuggingFaceH4/ultrafeedback\_binarized}} for our evaluations. The dataset provides scalar scores for the chosen and rejected responses, aggregated from multiple LLM evaluators. We compute the ground-truth margin for each preference as the score difference between the two responses, which also enables construction of PoP comparisons.
Although a preference dataset of size 
$\vert D_\text{pref}\vert$ can yield up to $\frac{\vert D_\text{pref}\vert (\vert D_\text{pref}\vert-1)}{2}$ PoP pairs, we restrict the PoP dataset to 
$\vert D_\text{PoP}\vert=k \vert D_\text{pref}\vert$ to maintain computational tractability. Appendix~\ref{ap:pop-data-scale} analyzes performance as a function of $k$ and shows that performance deteriorates beyond a certain value of $k$, justifying the use of small values; we use $k=2$ by default. We reiterate that there is NO quadratic scaling complexity due to pairing preferences, and it suffices to use a PoP dataset of size that is a small scalar multiple of the size of the preference dataset.

We evaluate two strategies for constructing the PoP dataset: one that represents each preference from the original dataset equally, and one that represents preferences in proportion to preference strength. We do this to explore the impact of different sampling strategies used to generate the PoP dataset, on downstream discriminative and generative performance. In the \textbf{iterative sampling} approach, each preference data point is equally represented by comparing it against $k$ weaker preferences. In practice, without ground-truth margin data, we could choose a preference and provide comparison preferences, asking the user for a label. We only choose $k$ preference pairs in which our chosen preference is judged to be stronger than the comparative preference. In contrast, the \textbf{random sampling} approach constructs the PoP dataset by randomly selecting pairs of preferences and labeling them based on their margins. This results in stronger preferences appearing more frequently in the PoP dataset than weaker ones. Furthermore, the \textbf{random sampling} approach is straightforward to implement in practice, in comparison to the iterative sampling approach, as this would only involve randomly sampling pairs of preferences and asking the annotator for a label. After generating the PoP dataset, we discard the original scalar scores and do not use them at any stage of model training.

\subsubsection{Experimental Setup}

We consider two models in our experiments: Llama-3.2-3b and Llama-3.1-8b \citep{grattafiori2024llama3herdmodels}. Following the standard direct alignment pipeline, we align these models using the UltraFeedback preference dataset \citep{cui2024ultrafeedbackboostinglanguagemodels}. We begin with a pretrained model and fine-tune it on the supervised fine-tuning (SFT) partition of the UltraFeedback dataset. Next, we align the models using the preference data from the same dataset. For further experimental details, refer to Appendix \ref{ap:exp_details}. Unless otherwise specified, all confidence intervals are computed over 6 seeds.

We evaluate the following variants of Direct Preference Optimization (DPO):

\begin{enumerate}[leftmargin=*, itemsep=0pt, topsep=0pt]
    \item \textbf{Vanilla DPO}: No margin is used in the loss function.
    \item \textbf{DPO-margin-1}: A fixed margin of 1 is applied to all preferences.
    \item \textbf{DPO-margin-gt}: Ground-truth margin values from the UltraFeedback dataset are used.
    \item \textbf{DPO-margin-gt-scaled}: This corresponds to the Scaled Bradley-Terry loss from \citet{wang2025helpsteer2preferencecomplementingratingspreferences}. The loss incorporates ground-truth margin information outside the log-sigmoid function rather than inside, effectively placing greater weight on preferences with larger margins. This can be interpreted as repeatedly sampling stronger preferences. The loss is defined as:
    % \begin{equation}
    %     \mathcal{L}_{\text{SBT}} = -m \log \sigma\left(\beta \log \frac{\pi_{\theta}(y^{+} \vert x)}{\pi_{\text{ref}}(y^{+} \vert x)} - \beta \log \frac{\pi_{\theta}(y^{-} \vert x)}{\pi_{\text{ref}}(y^{-} \vert x)}\right)
    % \end{equation}
    \begin{equation}
    \mathcal{L}_{\text{SBT}}(x,y_{+},y_{-}) = - m \log \sigma \!\left( \beta \log \frac{\pi_{\theta}(y^{+} \mid x)}{\pi_{\text{ref}}(y^{+} \mid x)} - \beta \log \frac{\pi_{\theta}(y^{-} \mid x)}{\pi_{\text{ref}}(y^{-} \mid x)} \right)
    \end{equation}

    \item \textbf{DPO-PoP-iter}: Margins are inferred from preference-over-preference (PoP) supervision (Eq.~\ref{eq:dpo-pop-clip}), using a PoP dataset constructed via iterative sampling.
    \item \textbf{DPO-PoP-random}: Margins are inferred from PoP supervision (Eq.~\ref{eq:dpo-pop-clip}), using a PoP dataset constructed via random sampling. This strategy can be interpreted as a bootstrapped version of the loss employed in DPO-margin-gt-scaled, along with a margin term (inside the log-sigmoid) that is inferred from preference-over-preference supervision.
\end{enumerate}

We provide the results for Llama-3.2-3b here. Results for Llama-3.1-8b are provided in Appendix \ref{ap:llama3.1-8b}.

\subsubsection{Discriminative Ability} \label{subsec: discriminative}
% \sn{It might be easier to visually parse this section if you separate each experiment into subsections, or have a bolded research question that each experiment is answering (or both).}

% Shortest
We evaluate DPO-PoP’s discriminative ability and margin correlation. For each preference $A \succ B$, we compare the UltraFeedback score difference (ground truth) with the DPO implicit reward difference (prediction). High correlation indicates better generalization and preference strength estimation. We report both Spearman and Pearson correlations. The correlation metrics are only possible in this setting due to access to UltraFeedback scores and cannot be computed in settings where ground-truth scores are unavailable; this analysis is provided purely for insight.

Table \ref{tab:dpo-margin-results-llama3.2-3b} shows that DPO-PoP-Iter attains the best test classification accuracy, outperforming the nearest baseline by \textbf{$\sim$8 percentage points}, and even outperforming DPO-margin-gt despite the latter having access to the true margin values.

The correlation metrics tell a different story: DPO-PoP-Random achieves the strongest Spearman and Pearson correlations, with DPO-PoP-Iter performing similarly on Spearman but substantially worse on Pearson. This suggests that DPO-PoP-Iter captures the correct ranking of preferences but its predicted margins are nonlinearly related to the true ones. We also see that DPO-PoP-Random exhibits lower accuracy but higher correlations overall. Figure \ref{fig:cum-accuracy-llama3.2} explains this tradeoff: DPO-PoP-Iter correctly classifies more weak-preference examples at the expense of stronger ones, whereas DPO-PoP-Random better captures strong preferences and is less influenced by noisy weak comparisons. As a result, DPO-PoP-Random maintains more faithful linear and ordinal relationships to the ground-truth margins, yielding superior Pearson and Spearman correlations.

% \begin{figure}  % 'r' for right-aligned figure
%     \centering
%     \subfloat[Lower Cumulative Accuracy vs Margin]{\includegraphics[width=60mm]{icml2026/figs/less_than_margin_acc.png}}
%     \subfloat[Upper Cumulative Accuracy vs Margin]{\includegraphics[width=60mm]{icml2026/figs/greater_than_margin_acc.png}}
%     \caption{Cumulative Accuracy vs Margin for the different DPO variants considered. Lower Cumulative Accuracy at margin $m$ indicates the accuracy of predicting preference labels using only datapoints with ground-truth margin less than or equal to $m$. Conversely, Upper Cumulative Accuracy reflects prediction accuracy on datapoints with ground-truth margin greater than or equal to $m$. The dark grey histogram shows the distribution (density) of margin values in the test set. In plot (a), DPO-PoP-Iter achieves higher accuracy on datapoints with lower margins, while in plot (b), its performance drops for higher margin datapoints. \blue{The lower cumulative accuracy plot is zoomed in, to address a reviewers request.}}
%     % This pattern suggests that DPO-PoP-Iter outperforms other variants, in terms of classification accuracy, due to the greater density of low-margin datapoints in the dataset, which it is able to label correctly.
% \label{fig:cum-accuracy-llama3.2}
% \end{figure}

\begin{figure*}[t]
    \centering
    \subfloat[Lower Cumulative Accuracy vs Margin]{%
        \includegraphics[width=0.48\textwidth, height=0.4\textwidth, keepaspectratio]{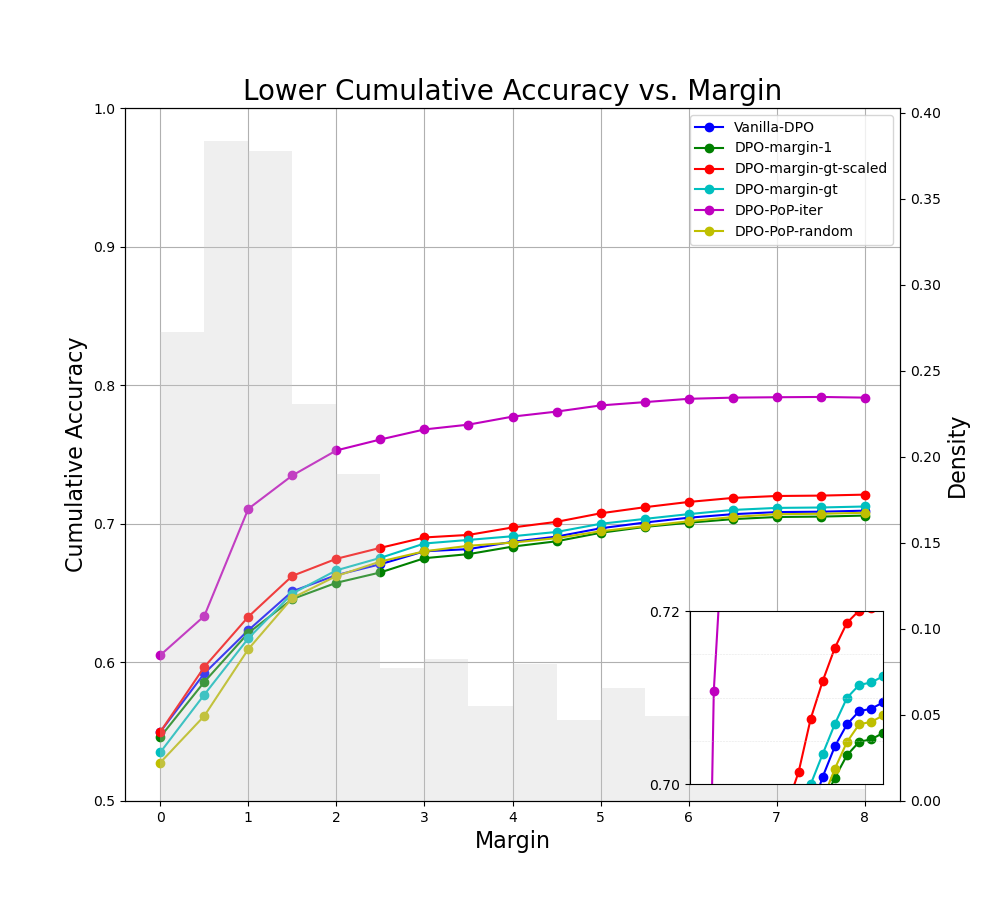}%
    }
    \hfill
    \subfloat[Upper Cumulative Accuracy vs Margin]{%
        \includegraphics[width=0.48\textwidth, height=0.4\textwidth, keepaspectratio]{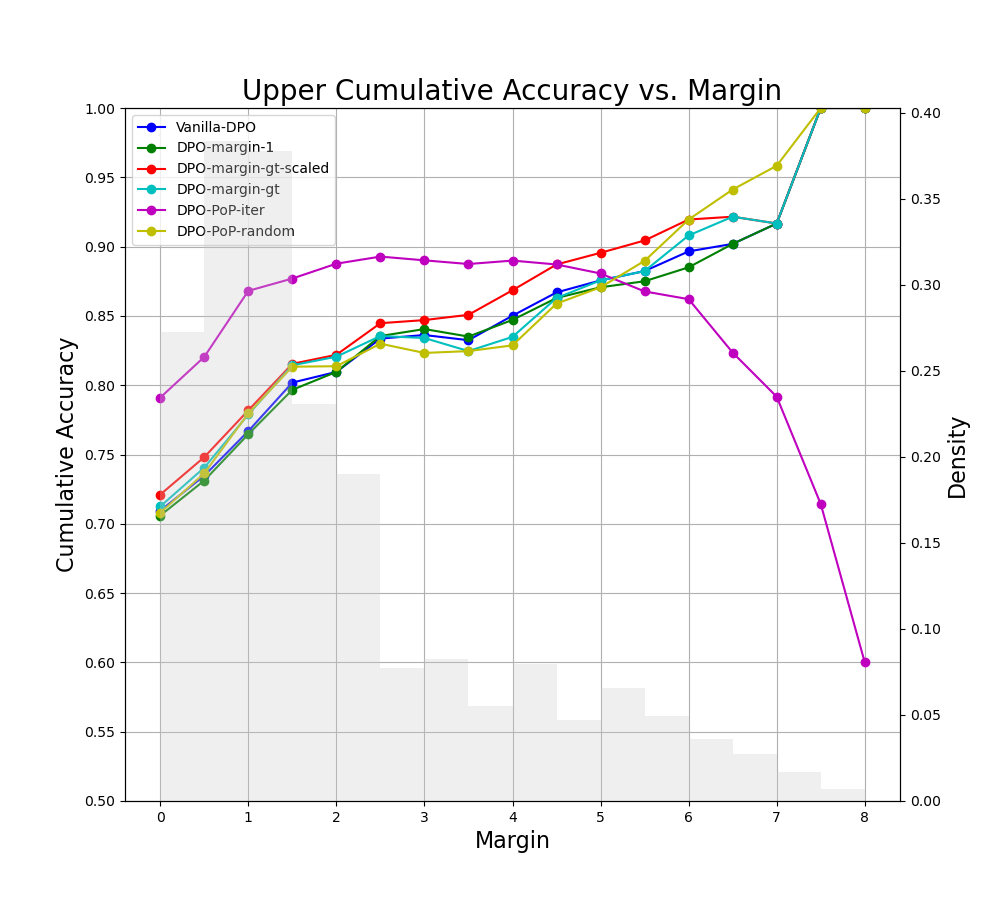}%
    }
    \caption{\small Cumulative Accuracy vs Margin for the different DPO variants considered. Lower Cumulative Accuracy at margin $m$ indicates the accuracy of predicting preference labels using only datapoints with ground-truth margin less than or equal to $m$. Conversely, Upper Cumulative Accuracy reflects prediction accuracy on datapoints with ground-truth margin greater than or equal to $m$. The dark grey histogram shows the distribution (density) of margin values in the test set. In plot (a), DPO-PoP-Iter achieves higher accuracy on datapoints with lower margins, while in plot (b), its performance drops for higher margin datapoints.}
    \label{fig:cum-accuracy-llama3.2}
\end{figure*}

We also report performance on RewardBench \citep{lambert2024rewardbenchevaluatingrewardmodels} in Table \ref{tab:dpo-rewardbench-llama3.2-3b}. DPO-PoP-Random outperforms all baselines in terms of the Overall score, improving over DPO-margin-gt (which has direct access to ground-truth margin values) by \textbf{1.4 percentage points}. Notably, DPO-PoP-iter heavily outperforms all methods on the Chat split but also strongly underperforms on the Reasoning split—which comprises a larger portion of the dataset—resulting in a lower Overall score compared to DPO-PoP-random. In contrast, DPO-PoP-random delivers stable performance across all categories, securing the highest Overall score.

\newcommand{\meanstd}[2]{#1{\scriptsize$\;\pm\;$#2}}

% \begin{table*}[t]
% \centering
% \caption{Comparison of DPO variants on classification accuracy and Spearman, Pearson correlation with ground-truth margins for Llama-3.2-3b, when using synthetic data.}
% \label{tab:dpo-margin-results-llama3.2-3b}
% \begin{tabular}{lccc}
% \toprule
% \textbf{Algorithm} & \textbf{Pearson Correlation} & \textbf{Spearman Correlation} & \textbf{Accuracy (\%)} \\
% \midrule
% Vanilla-DPO 
% & \meanstd{0.2940}{0.0036}
% & \meanstd{0.3003}{0.0036}
% & \meanstd{71.15}{0.178} \\

% DPO-margin-1 
% & \meanstd{0.2929}{0.0041}
% & \meanstd{0.2984}{0.0045}
% & \meanstd{71.18}{0.28} \\

% DPO-margin-gt 
% & \meanstd{0.3427}{0.0029}
% & \meanstd{0.3451}{0.0028}
% & \meanstd{71.85}{0.34} \\

% DPO-margin-gt-scaled 
% & \meanstd{0.3381}{0.0037}
% & \meanstd{0.3453}{0.0033}
% & \meanstd{72.05}{0.16} \\

% DPO-PoP-iter
% & \meanstd{0.2449}{0.0017}
% & \underline{\textbf{\boldmath \meanstd{0.3656}{0.0008}}}
% & \underline{\textbf{\boldmath \meanstd{79.97}{0.41}}} \\

% DPO-PoP-random 
% & \underline{\textbf{\boldmath \meanstd{0.3639}{0.0020}}}
% & \underline{\textbf{\boldmath \meanstd{0.3685}{0.0010}}}
% & \meanstd{71.09}{0.21} \\
% \bottomrule
% \end{tabular}
% \end{table*}

\begin{table}[t]  % single-column, force placement
\centering
\footnotesize       % slightly bigger than scriptsize
\setlength{\tabcolsep}{2.0pt}  % tighten spacing
\caption{Comparison of DPO variants on classification accuracy and Spearman, Pearson correlation with ground-truth margins for Llama-3.2-3b using synthetic data.}
\label{tab:dpo-margin-results-llama3.2-3b}
\begin{tabular}{p{2.75cm}ccc}
\toprule
\textbf{Algorithm} & \textbf{Pearson} & \textbf{Spearman} & \textbf{Accuracy\%} \\
\midrule
Vanilla-DPO          & \meanstd{0.2940}{0.004} & \meanstd{0.3003}{0.004} & \meanstd{71.15}{0.18} \\
DPO-margin-1         & \meanstd{0.2929}{0.004} & \meanstd{0.2984}{0.005} & \meanstd{71.18}{0.28} \\
DPO-margin-gt        & \meanstd{0.3427}{0.003} & \meanstd{0.3451}{0.003} & \meanstd{71.85}{0.34} \\
DPO-margin-gt-scaled & \meanstd{0.3381}{0.004} & \meanstd{0.3453}{0.003} & \meanstd{72.05}{0.16} \\
DPO-PoP-iter         & \meanstd{0.2449}{0.002} & \underline{\textbf{\boldmath \meanstd{0.3656}{0.001}}} & \underline{\textbf{\boldmath \meanstd{79.97}{0.41}}} \\
DPO-PoP-random       & \underline{\textbf{\boldmath \meanstd{0.3639}{0.002}}} & \underline{\textbf{\boldmath \meanstd{0.3685}{0.001}}} & \meanstd{71.09}{0.21} \\
\bottomrule
\end{tabular}
\end{table}

\begin{table*}[t]
\centering
\footnotesize
\caption{Performance of Llama-3.2-3b DPO variants on RewardBench, when using synthetic data. Higher is better.}
\label{tab:dpo-rewardbench-llama3.2-3b}
\begin{tabular}{lccccc}
\toprule
\textbf{Algorithm} & \textbf{Chat} & \textbf{Chat-Hard} & \textbf{Safety} & \textbf{Reason} & \textbf{Overall} \\
\midrule
Vanilla-DPO 
& \meanstd{75.65}{0.34} 
& \underline{\textbf{\boldmath \meanstd{64.51}{0.51}}} 
& \meanstd{71.49}{0.17} 
& \meanstd{75.85}{0.46} 
& \meanstd{75.46}{0.21} \\

DPO-margin-1 
& \meanstd{76.86}{0.54} 
& \meanstd{64.14}{0.21} 
& \meanstd{71.19}{0.86} 
& \meanstd{77.03}{0.23} 
& \meanstd{75.78}{0.29} \\

DPO-margin-gt 
& \meanstd{80.35}{0.38} 
& \meanstd{63.27}{0.21} 
& \meanstd{75.70}{0.31} 
& \underline{\textbf{\boldmath \meanstd{78.05}{0.47}}} 
& \meanstd{77.45}{0.25} \\

DPO-margin-gt-scaled 
& \meanstd{80.87}{0.55} 
& \meanstd{64.11}{0.53} 
& \meanstd{75.47}{0.46} 
& \meanstd{76.33}{0.27} 
& \meanstd{77.13}{0.29} \\

DPO-PoP-iter
& \underline{\textbf{\boldmath \meanstd{87.71}{0.53}}} 
& \meanstd{59.61}{0.50} 
& \meanstd{81.28}{0.62} 
& \meanstd{69.83}{1.35} 
& \meanstd{76.73}{0.24} \\

DPO-PoP-random 
& \meanstd{82.73}{0.80} 
& \meanstd{62.54}{0.63} 
& \underline{\textbf{\boldmath \meanstd{81.94}{1.07}}} 
& \meanstd{76.44}{0.69} 
& \underline{\textbf{\boldmath \meanstd{78.87}{0.25}}} \\
\bottomrule
\end{tabular}
\end{table*}

% \begin{table}[h]
% \centering
% \begin{tabular}{lccccc}
% \toprule
% \textbf{Model} & \textbf{Chat} & \textbf{Chat Hard} & \textbf{Safety} & \textbf{Reasoning} & \textbf{Overall} \\
% \midrule
% Vanilla-DPO & 75.14 & \underline{\textbf{64.69}} & 71.22 & 76.25 & 75.51 \\
% DPO-margin-1 & 77.65 & 64.25 & 72.57 & 76.83 & 76.18 \\
% DPO-margin-gt & 80.17 & 63.60 & 75.54 & 77.38 & 77.12 \\
% DPO-margin-gt-scaled & 80.45 & 63.60 & 75.54 & 76.07 & 76.85 \\
% DPO-PoP-iter & \underline{\textbf{87.99}} & 59.21 & \underline{\textbf{80.14}} & 72.08 & 77.09 \\
% DPO-PoP-random & 81.01 & 62.72 & 79.59 & \underline{\textbf{77.75}} & \underline{\textbf{78.66}} \\
% \bottomrule
% \end{tabular}
% \caption{Performance of Llama-3.2-3b DPO variants on RewardBench. Higher is better.}
% \label{tab:dpo-rewardbench-llama3.2-3b}
% \end{table}

\subsubsection{Generative Ability}
We evaluate generative ability using UltraRM \citep{cui2024ultrafeedbackboostinglanguagemodels}, a larger reward model trained on UltraFeedback, as a gold RM judge similar in spirit to \citet{gao2022scalinglawsrewardmodel}. To validate UltraRM as a proxy, we measured its agreement with GPT-4.1-mini on 512 prompts (bootstrapped over 100k samples), obtaining 69.03\% $\pm$ 2.18\% agreement. We report win rate and median advantage over Vanilla-DPO (reference model), where advantage is the per-prompt UltraRM reward difference between each evaluated model and the reference model. Table~\ref{tab:dpo-gen-results-llama3.2-3b} shows that DPO-PoP-random outperforms all baselines, including those with access to ground-truth margins, in terms of win rate. AlpacaEval 2.0 results (Table~\ref{tab:dpo-alpaca-llama3.2-3b}) are consistent: DPO-PoP-random outperforms all baselines, achieving \textbf{14.2\% LC win rate} vs 12.4\% for its nearest competitor, DPO-margin-gt.

% \begin{table}[h!]
% \centering
% % helper macro to shrink std dev
% \begin{tabular}{lcc}
% \toprule
% \textbf{Method} & \textbf{Median Advantage} & \textbf{Win Rate (\%)} \\
% \midrule
% DPO-margin-1
%   & \meanstd{0.2272}{0.0202}
%   & \meanstd{54.91}{0.34} \\

% DPO-margin-gt
%   & \meanstd{0.5863}{0.0577}
%   & \meanstd{61.25}{1.15} \\

% DPO-margin-gt-scaled
%   & \meanstd{0.1602}{0.0284}
%   & \meanstd{53.65}{0.64} \\

% DPO-PoP-iter
%   & \meanstd{0.3887}{0.0452}
%   & \meanstd{57.76}{0.88} \\
  
% DPO-PoP-random 
%   & \underline{\textbf{\boldmath \meanstd{0.6745}{0.0506}}}
%   & \underline{\textbf{\boldmath \meanstd{62.39}{1.12}}} \\
  
% \bottomrule
% \end{tabular}
% \caption{Comparison of margin-based DPO variants against Vanilla DPO on median advantage and win rate for Llama-3.2-3b.  \blue{This table was modified to include confidence intervals over 6 seeds (including the earlier result) to address the reviewers' questions during the rebuttals.}}
% \label{tab:dpo-gen-results-llama3.2-3b}
% \end{table}

\begin{table*}[t]
\centering
\footnotesize
\caption{Generative performance of Llama-3.2-3b DPO variants using synthetic data.}
\label{tab:dpo-gen-synthetic}
\begin{subtable}[c]{0.44\textwidth}
\centering
\setlength{\tabcolsep}{3pt}
\caption{UltraRM win rate and median advantage.}
\label{tab:dpo-gen-results-llama3.2-3b}
\begin{tabular}{lcc}
\toprule
\textbf{Algorithm} & \textbf{Median Adv} & \textbf{Win Rate\%} \\
\midrule
DPO-margin-1         & 0.227 $\pm$ 0.020 & 54.9 $\pm$ 0.34 \\
DPO-margin-gt        & 0.586 $\pm$ 0.058 & 61.3 $\pm$ 1.15 \\
DPO-margin-gt-scaled & 0.160 $\pm$ 0.028 & 53.7 $\pm$ 0.64 \\
DPO-PoP-iter         & 0.389 $\pm$ 0.045 & 57.8 $\pm$ 0.88 \\
DPO-PoP-random       & \underline{\textbf{0.675 $\pm$ 0.051}} & \underline{\textbf{62.4 $\pm$ 1.12}} \\
\bottomrule
\end{tabular}
\end{subtable}
\hfill
\begin{subtable}[c]{0.54\textwidth}
\centering
\setlength{\tabcolsep}{2pt}
\caption{AlpacaEval 2.0 win rate.}
\label{tab:dpo-alpaca-llama3.2-3b}
\begin{tabular}{p{2.2cm}ccc}
\toprule
\textbf{Algorithm} & \textbf{LC-Win Rate\%} & \textbf{Win Rate\%} & \textbf{Avg Len} \\
\midrule
Vanilla-DPO          & \meanstd{11.74}{0.74} & \meanstd{11.37}{0.69} & \meanstd{1800}{17} \\
DPO-margin-1         & \meanstd{11.74}{1.04} & \meanstd{11.51}{1.04} & \meanstd{1823}{29} \\
DPO-margin-gt        & \meanstd{12.40}{0.71} & \meanstd{12.17}{0.58} & \meanstd{1915}{42} \\
DPO-margin-gt-scaled & \meanstd{10.99}{0.79} & \meanstd{10.97}{0.71} & \meanstd{1836}{19} \\
DPO-PoP-iter         & \meanstd{12.30}{0.70} & \meanstd{12.26}{0.62} & \meanstd{1919}{50} \\
DPO-PoP-random       & \underline{\textbf{\boldmath \meanstd{14.24}{1.06}}} & \underline{\textbf{\boldmath \meanstd{13.69}{1.02}}} & \meanstd{1846}{20} \\
\bottomrule
\end{tabular}
\end{subtable}
\end{table*}

In Tables~\ref{tab:dpo-gen-results-llama3.2-3b} and~\ref{tab:dpo-alpaca-llama3.2-3b}, we observe that DPO-PoP-iter underperforms compared to DPO-PoP-random and DPO-margin-gt. We hypothesize that this is due to correctly classifying weaker preferences at the expense of stronger preferences, as discussed in Section \ref{subsec: discriminative}. By potentially overfitting to noisy weaker preferences, DPO-PoP-iter suffers a drop in generative performance.

\subsection{LLM Annotated Preference over Preference Data Experiments}

Instead of using UltraFeedback scores to construct PoP labels, we obtain them directly from GPT-4.1-mini, serving as a test bed for realistic annotation settings. To keep annotation costs low, we randomly sample 5,000 preferences from UltraFeedback to train all models and construct the PoP dataset by randomly pairing preferences and querying the LLM ($k=2$, Llama-3.2-3b). Ties are discarded from the final PoP dataset. We focus on DPO-PoP-Random as its annotations are easier to obtain in practice than those required for the iterative variant. Additional experiments on robustness to annotator noise are in Appendix~\ref{ap:pop-noise}; the prompt used to gather PoP annotations is in Appendix~\ref{ap:pop-prompt}.

\subsubsection{Discriminative Performance}
The results showing the test classification accuracy on the UltraFeedback dataset \citep{cui2024ultrafeedbackboostinglanguagemodels} and RewardBench \citep{lambert2024rewardbenchevaluatingrewardmodels} scores are in Tables \ref{tab:dpo-margin-results-llama3.2-3b-llm-annotated} and \ref{tab:dpo-rewardbench-llama3.2-3b-llm-annotated} respectively. On test classification accuracy (Table~\ref{tab:dpo-margin-results-llama3.2-3b-llm-annotated}), DPO-margin-gt-scaled has an edge over other approaches. On RewardBench (Table~\ref{tab:dpo-rewardbench-llama3.2-3b-llm-annotated}), all approaches perform similarly in terms of Overall score.

% \begin{table}[h]
% \centering
% \begin{tabular}{lccc}
% \toprule
% \textbf{Algorithm} & \textbf{Pearson Correlation} & \textbf{Spearman Correlation} & \textbf{Accuracy} \\
% \midrule
% Vanilla DPO & 0.1180 & 0.1427 & 0.63 \\
% DPO-margin-1 & 0.1037 & 0.1276 & 0.61 \\
% DPO-margin-gt &  0.1040 & 0.1237 & 0.61 \\
% DPO-margin-gt-scaled & \underline{\textbf{0.1486}} & \underline{\textbf{0.1712}} & \underline{\textbf{0.64}} \\
% DPO-PoP-random & 0.1406 & \underline{\textbf{0.1649}} & 0.63 \\
% \bottomrule
% \end{tabular}
% \caption{Comparison of DPO variants on classification accuracy and Spearman, Pearson correlation with ground-truth margins for Llama-3.2-3b. The PoP labels for DPO-PoP-Random are obtained from a GPT-4.1-mini annotated Preference-over-Preference dataset. \blue{This table was newly added to address the reviewers' questions during the rebuttals.}}
% \label{tab:dpo-margin-results-llama3.2-3b-llm-annotated}
% \end{table}

\begin{table}[t]
\centering
\footnotesize
\setlength{\tabcolsep}{2pt} % tighten spacing
\caption{Comparison of DPO variants on classification accuracy and Spearman/Pearson correlation with ground-truth margins for Llama-3.2-3b, when using LLM annotated PoP data.}
\label{tab:dpo-margin-results-llama3.2-3b-llm-annotated}
\begin{tabular}{p{2.75cm}ccc} 
\toprule
\textbf{Algorithm} & \textbf{Pearson} & \textbf{Spearman} & \textbf{Accuracy\%} \\
\midrule
Vanilla DPO          & \meanstd{0.1171}{0.001} & \meanstd{0.1435}{0.001} & \meanstd{62.84}{0.17} \\
DPO-margin-1         & \meanstd{0.1035}{0.002} & \meanstd{0.1276}{0.003} & \meanstd{61.48}{0.40} \\
DPO-margin-gt        & \meanstd{0.1045}{0.003} & \meanstd{0.1243}{0.004} & \meanstd{61.33}{0.25} \\
DPO-margin-gt-scaled & \underline{\textbf{\boldmath\meanstd{0.1505}{0.002}}} & \underline{\textbf{\boldmath\meanstd{0.1727}{0.003}}} & \underline{\textbf{\boldmath\meanstd{64.02}{0.10}}} \\
DPO-PoP-random       & \underline{\textbf{\boldmath\meanstd{0.1408}{0.001}}} & \underline{\textbf{\boldmath\meanstd{0.1645}{0.001}}} & \meanstd{62.98}{0.05} \\
\bottomrule
\end{tabular}
\end{table}

\begin{table*}[t]
\centering
\footnotesize
\setlength{\tabcolsep}{4pt} % slightly tighter spacing
\caption{Performance of Llama-3.2-3b DPO variants on RewardBench, when using LLM annotated PoP data. Higher is better. All approaches achieve similar Overall performance. DPO-PoP-Random outperforms on Reasoning, and DPO-margin-gt-scaled outperforms on Chat. The boldened and underlined values have $>$ 1 percentage point difference with the nearest competitor.}
\label{tab:dpo-rewardbench-llama3.2-3b-llm-annotated}
\begin{tabular}{lccccc}
\toprule
\textbf{Algorithm} & \textbf{Chat} & \textbf{Chat Hard} & \textbf{Safety} & \textbf{Reasoning} & \textbf{Overall} \\
\midrule
Vanilla-DPO 
& \meanstd{64.25}{1.11} & \meanstd{62.24}{0.53} & \meanstd{64.05}{0.55} & \meanstd{80.90}{0.76} & \meanstd{73.05}{0.56} \\

DPO-margin-1 
& \meanstd{61.87}{0.75} & \meanstd{63.38}{0.59} & \meanstd{63.67}{0.38} & \meanstd{83.00}{0.14} & \meanstd{73.49}{0.17} \\

DPO-margin-gt 
& \meanstd{61.08}{0.26} & \meanstd{63.85}{0.77} & \meanstd{63.74}{0.34} & \meanstd{84.01}{0.53} & \meanstd{73.77}{0.28} \\

DPO-margin-gt-scaled 
& \underline{\textbf{\boldmath\meanstd{67.83}{1.05}}} & \meanstd{62.21}{0.35} & \meanstd{64.89}{0.75} & \meanstd{80.32}{0.73} & \meanstd{73.51}{0.54} \\

DPO-PoP-random 
& \meanstd{58.10}{0.88} & \meanstd{63.60}{0.51} & \meanstd{62.48}{0.64} & \underline{\textbf{\boldmath\meanstd{85.19}{0.44}}} & \meanstd{73.57}{0.36} \\
\bottomrule
\end{tabular}
\end{table*}

\subsubsection{Generative Performance}

Tables~\ref{tab:dpo-gen-results-llama3.2-3b-llm-annotated} and~\ref{tab:dpo-alpaca-llama3.2-3b-llm-annotated} show UltraRM win rates and AlpacaEval 2.0 results respectively; DPO-PoP-Random outperforms all other baselines in generative quality across both evaluations, improving over its nearest competitor, DPO-margin-gt, by \textbf{over 5 percentage points} in UltraRM win rate.

% \begin{table}[H]
% \centering
% \begin{tabular}{lcc}
% \toprule
% \textbf{Method} & \textbf{Median Advantage} & \textbf{Win Rate (\%)} \\
% \midrule
% DPO-margin-1 & 0.1719 & 54\% \\
% DPO-margin-gt & 0.3750 & 58\% \\
% DPO-margin-gt-scaled & 0.0938 & 53\% \\
% DPO-PoP-Random & \underline{\textbf{0.9375}} & \underline{\textbf{65\%}} \\
% \bottomrule
% \end{tabular}
% \caption{
% Comparison of margin-based DPO variants on median advantage and win rate for Llama-3.2-3B. 
% The PoP labels for DPO-PoP-Random are obtained from a GPT-4.1-mini annotated Preference-over-Preference dataset. \blue{This table was newly added to address the reviewers' questions during the rebuttals.}
% }
% \label{tab:dpo-gen-results-llama3.2-3b-llm-annotated}
% \end{table}

\begin{table*}[t]
\centering
\footnotesize
\caption{Generative performance of Llama-3.2-3b DPO variants using LLM annotated PoP data.}
\label{tab:dpo-gen-llm-annotated}
\begin{subtable}[c]{0.40\textwidth}
\centering
\setlength{\tabcolsep}{3pt}
\caption{UltraRM win rate and median advantage.}
\label{tab:dpo-gen-results-llama3.2-3b-llm-annotated}
\begin{tabular}{lcc}
\toprule
\textbf{Algorithm} & \textbf{Median Adv} & \textbf{Win Rate\%} \\
\midrule
DPO-margin-1         & \meanstd{0.138}{0.058} & \meanstd{53.1}{1.1} \\
DPO-margin-gt        & \meanstd{0.347}{0.095} & \meanstd{56.9}{1.5} \\
DPO-margin-gt-scaled & \meanstd{0.044}{0.034} & \meanstd{50.7}{1.8} \\
DPO-PoP-Random       & \underline{\textbf{\boldmath\meanstd{0.741}{0.113}}} & \underline{\textbf{\boldmath\meanstd{62.2}{1.8}}} \\
\bottomrule
\end{tabular}
\end{subtable}
\hfill
\begin{subtable}[c]{0.57\textwidth}
\centering
\setlength{\tabcolsep}{2pt}
\caption{AlpacaEval 2.0 win rate.}
\label{tab:dpo-alpaca-llama3.2-3b-llm-annotated}
\begin{tabular}{p{2.2cm}ccc}
\toprule
\textbf{Algorithm} & \textbf{LC-Win Rate\%} & \textbf{Win Rate\%} & \textbf{Avg Len} \\
\midrule
Vanilla-DPO          & \meanstd{7.66}{0.98}  & \meanstd{6.38}{0.69}  & \meanstd{1455}{27} \\
DPO-margin-1         & \meanstd{8.37}{0.81}  & \meanstd{6.93}{0.65}  & \meanstd{1476}{17} \\
DPO-margin-gt        & \meanstd{10.40}{1.09} & \meanstd{8.68}{0.87}  & \meanstd{1518}{31} \\
DPO-margin-gt-scaled & \meanstd{8.09}{0.64}  & \meanstd{6.70}{0.47}  & \meanstd{1469}{24} \\
DPO-PoP-random       & \underline{\textbf{\boldmath\meanstd{11.07}{1.20}}} & \underline{\textbf{\boldmath\meanstd{9.39}{1.19}}} & \meanstd{1580}{29} \\
\bottomrule
\end{tabular}
\end{subtable}
\end{table*}

% \begin{table}[H]
% \centering
% \begin{tabular}{lccc}
% \toprule
% \textbf{Experiment} & \textbf{Length-Controlled Win Rate} & \textbf{Win Rate} & \textbf{Avg Length} \\
% \midrule
% Vanilla-DPO & 8.85 & 7.33 & 1507 \\
% DPO-margin-1 & 9.47 & 7.95 & 1508 \\
% DPO-margin-gt & 11.78 & 9.94 & 1573 \\
% DPO-margin-gt-scaled & 8.25 & 6.83 & 1506 \\
% DPO-PoP-random & \underline{\textbf{12.40}} & \underline{\textbf{10.93}} & 1630 \\
% \bottomrule
% \end{tabular}
% \caption{Performance of Llama-3.2-3b DPO variants on the AlpacaEval 2.0 benchmark. The PoP labels for DPO-PoP-Random are obtained from a GPT-4.1-mini annotated Preference-over-Preference dataset. \blue{This table was newly added to address the reviewers' questions during the rebuttals.}}
% \label{tab:dpo-alpaca-llama3.2-3b-llm-annotated}
% \end{table}

% \begin{table}[t]
% \centering
% \setlength{\tabcolsep}{1.5pt} % tighter column spacing
% \caption{Performance of Llama-3.2-3b DPO variants on the AlpacaEval 2.0 benchmark, when using LLM annotated PoP data.}
% \label{tab:dpo-alpaca-llama3.2-3b-llm-annotated}
% \begin{tabular}{lccc}
% \toprule
% \textbf{Algorithm} & \textbf{LC-Win Rate} & \textbf{Win Rate} & \textbf{Avg Length} \\
% \midrule
% Vanilla-DPO          & 8.85  & 7.33  & 1507 \\
% DPO-margin-1         & 9.47  & 7.95  & 1508 \\
% DPO-margin-gt        & 11.78 & 9.94  & 1573 \\
% DPO-margin-gt-scaled & 8.25  & 6.83  & 1506 \\
% DPO-PoP-random       & \underline{\textbf{12.40}} & \underline{\textbf{10.93}} & 1630 \\
% \bottomrule
% \end{tabular}
% \end{table}

\subsection{Discrimination vs Generation}

% You should have a small subsection at the end here talking about the tradeoffs between generative and discriminative and how our inisghts make it possible to make these tradeoffs in an informed way or something like that (although I'm still not clear on how reasonable the iterative version is to do with real human data, so maybe that points to our method being more easily applicable to the generative case?)

Across all experiments, we observed that \emph{\uline{DPO-PoP-Random consistently improves generative performance and DPO-PoP-Iter consistently improves in-domain discriminative performance}}, reflecting a trade-off between the two objectives. Practitioners can choose the approach that best matches their use case. DPO-PoP-iter achieves good discriminative performance on in-distribution test data (Table~\ref{tab:dpo-margin-results-llama3.2-3b}), but overfitting to weaker preferences comes at the cost of generative quality. DPO-PoP-random achieves good generative performance (Tables~\ref{tab:dpo-gen-synthetic} and~\ref{tab:dpo-gen-llm-annotated}) and is also robust in terms of discriminative performance, as supported by the RewardBench results in Table~\ref{tab:dpo-rewardbench-llama3.2-3b}. We provide a discussion of this discriminative-generative tradeoff in Appendix ~\ref{ap: discussion} with corresponding theory in Appendix ~\ref{ap:theorem}. From Theorem ~\ref{thm:main}, we observe that the bound on the generalization performance of an adaptive-margin classifier depends on $\widetilde{M} = \sqrt{\sum_{i=1}^{N} \frac{1}{m_{i}^2}}$ , where $m_i$ are the margins of training data points. DPO-PoP-Random achieves a smaller $\widetilde{M}$, since it does not incorporate the large fraction of weak-preference (small-margin) comparisons included in DPO-PoP-Iter. As a result, its generalization error bound is lower, particularly when train–test distributions do not overlap substantially.

\subsection{Small-dataset setting}
DPO-PoP-Random achieves a \textbf{17 percentage point} improvement in win rate over DPO-margin-gt in the low-data regime, despite the latter having access to ground-truth margins from the UltraFeedback dataset (Appendix~\ref{ap:pop-data-scale}). This suggests that gathering PoP annotations is a simple and effective way to boost generative performance, particularly when working with small preference datasets.

\section{Related Work}
\label{sec: related_work}
Techniques that employ margins have largely been employed in the reward modeling phase of the RLHF pipeline. \citet{touvron2023llama} used margins derived from preference ratings given by human annotators, in order to train reward models, and showed that the margin term can help the helpfulness reward model accuracy, especially when the two responses are more separable. \citet{wang2025helpsteer2preferencecomplementingratingspreferences} propose Scaled Bradley-Terry loss, a margin based reward modeling objective that uses the margins derived from preference ratings in order to scale the loss for each datapoint. This can be seen as upsampling preferences for which the margin is higher. They show that the scaled loss variant leads to better performance than the margin loss variant proposed in \citet{touvron2023llama}. \citet{wang2024rewarddifferenceoptimizationsample} propose Reward Difference Optimization, that also uses a scaled loss, but uses margins computed from a learned reward model to scale each data point. %Other approaches, compute margins to be the average difference between the rewards of the chosen and rejected responses in a training batch \citep{qin2024understandinginfluencerewardmargin}, or use an ensemble of trained reward models to compute the average difference between the rewards of the chosen and rejected responses in a preference, over the ensemble and use it as the margin \citep{wang2024secretsrlhflargelanguage}. 
DPO-PoP-random can be interpreted as a bootstrapped variant of the Scaled Bradley-Terry loss~\citep{wang2025helpsteer2preferencecomplementingratingspreferences, wang2024rewarddifferenceoptimizationsample}.
Other approaches compute margins in different ways. \cite{qin2024understandinginfluencerewardmargin} define the margin as the average difference between the rewards of the chosen and rejected responses within each training batch. \cite{wang2024secretsrlhflargelanguage} use an ensemble of reward models and calculate the margin as the average reward difference across the ensemble for each preference.
% The latter approach \citep{wang2024secretsrlhflargelanguage} is an adaptive margin formulation, but requires training an ensemble of reward models to compute the margin values.

In the case of Direct Alignment Algorithms \citep{rafailov2024scalinglawsrewardmodel}, IPO \citep{azar2023generaltheoreticalparadigmunderstand} and SLiC \citep{zhao2023slic} can also be interpreted in terms of margin, wherein IPO regresses the difference of implicit rewards to a fixed margin, whereas SLiC uses hinge loss with a fixed margin. \citet{amini2024direct}, propose ODPO, which is a variant of DPO with an offset. They use a reward model to label the preference data and also to provide the margin values to be used in the ODPO loss. Another approach, $\alpha$-DPO \citep{wu2024alphadpoadaptiverewardmargin}, redefines the reference policy $\hat{\pi}_\text{ref}$ , to blend between the policy $\pi$ and the reference policy $\pi_\text{ref}$, to achieve personalized reward margins. \cite{wu2024betadpodirectpreferenceoptimization} observe that the optimal $\beta$ value for the DPO loss depends on the informativeness of the pairwise preference data, and they propose $\beta$-DPO, which dynamically calibrates $\beta$ at the batch level based on data quality. DPO-PoP instead infers margins from preference-over-preference annotations.

\section{Conclusion}

\label{sec: conclusion}
% Again, we'll want to have a nice final takeaway here about how we've shown how to improve RLHF and guide practitioners to choose how to implement sampling in PoP, depending on their use case.

% We introduced DPO-PoP, a framework that incorporates adaptive margins into the DPO loss using preference-over-preference (PoP) supervision. Unlike prior approaches that derive margins from scalar preference ratings—either provided by annotators or learned via reward models—DPO-PoP infers margins directly from ordinal comparisons between preferences. We propose two sampling strategies to generate PoP data: random and iterative. Our results indicate that attempting to improve discriminative performance by correctly predicting weaker preference labels, as seen with DPO-PoP-iter, comes at the cost of generative performance. In contrast, we show that DPO-PoP-random achieves superior generative performance compared to DPO variants that use fixed or no margins, or margins computed from preference scores, and also offers robust discriminative performance, as evidenced by results on RewardBench.

We introduced DPO-PoP, a framework that integrates adaptive margins into DPO via preference-over-preference (PoP) supervision. Unlike prior approaches that derive margins from scalar preference ratings, which can be noisy and inconsistently calibrated across annotators, DPO-PoP infers margins directly from ordinal comparisons between preferences. We explored two PoP sampling strategies: random and iterative. Our results reveal a tradeoff: DPO-PoP-iter improves in-domain discriminative performance at the expense of generative quality, whereas DPO-PoP-random achieves better generative performance than DPO baselines using fixed or score-derived margins and maintains robust discriminative accuracy on RewardBench. This tradeoff allows practitioners to choose the approach that best matches their use case. We demonstrate that PoP supervision provides fine-grained margin information through simple comparative judgments, offering a practical alternative to numerical rating-scale annotations, and yielding consistent improvements in discriminative and generative performance when the appropriate PoP variant is selected for the target use case.

\section*{Acknowledgements}

This work has taken place in the Safe, Correct, and Aligned Learning and Robotics Lab (SCALAR) at The University of Massachusetts Amherst. SCALAR research is supported in part by the NSF (IIS-2437426) and Open Philanthropy. The computational resources for this work were provided by the University of Massachusetts Amherst's partnership with the Unity Research Computing Platform, a multi-institutional cluster led by the University of Massachusetts and the University of Rhode Island.

Scott Niekum holds concurrent appointments as an Associate Professor at the University of Massachusetts Amherst and as an Amazon Scholar. This paper describes work performed at the University of Massachusetts Amherst and is not associated with Amazon.

\bibliographystyle{plainnat}
\bibliography{neurips_2026}

% \section*{References}

% References follow the acknowledgments in the camera-ready paper. Use unnumbered first-level heading for
% the references. Any choice of citation style is acceptable as long as you are
% consistent. It is permissible to reduce the font size to \verb+small+ (9 point)
% when listing the references.
% Note that the Reference section does not count towards the page limit.
% \medskip

% {
% \small

% [1] Alexander, J.A.\ \& Mozer, M.C.\ (1995) Template-based algorithms for
% connectionist rule extraction. In G.\ Tesauro, D.S.\ Touretzky and T.K.\ Leen
% (eds.), {\it Advances in Neural Information Processing Systems 7},
% pp.\ 609--616. Cambridge, MA: MIT Press.

% [2] Bower, J.M.\ \& Beeman, D.\ (1995) {\it The Book of GENESIS: Exploring
%   Realistic Neural Models with the GEneral NEural SImulation System.}  New York:
% TELOS/Springer--Verlag.

% [3] Hasselmo, M.E., Schnell, E.\ \& Barkai, E.\ (1995) Dynamics of learning and
% recall at excitatory recurrent synapses and cholinergic modulation in rat
% hippocampal region CA3. {\it Journal of Neuroscience} {\bf 15}(7):5249-5262.
% }

%%%%%%%%%%%%%%%%%%%%%%%%%%%%%%%%%%%%%%%%%%%%%%%%%%%%%%%%%%%%
\newpage
\appendix
\onecolumn
\label{sec: appendix}
% \section{Large Language Model Usage}

% Large Language Models (LLMs) were used solely for grammatical editing and improving writing flow. The research methodology, experimental design, data analysis, and all scientific conclusions are entirely the work of the human authors.

\section{Impact Statement and Limitations}\label{ap:impact}

This work proposes a method for improving the alignment of large language models (LLMs) with human preferences. Better-aligned LLMs have broad positive societal impacts, including more helpful and reliable AI assistants. By enabling more fine-grained margin information through simple comparative annotations, our approach reduces reliance on numerical scores, which can be inconsistently calibrated, and thus supports more robust alignment pipelines.

As with any work that improves LLM capabilities, there is a risk that better-aligned models could be misused to generate more persuasive misinformation or harmful content. However, this risk is not specific to our contribution and applies broadly to LLM alignment research. We do not foresee direct negative societal impacts specific to the proposed annotation and training methodology.

\textbf{Limitations.} The theoretical generalization bound (Appendix~\ref{ap:theorem}) assumes a linear reward function, whereas practical reward models are large neural networks. This is a standard assumption in the analysis of generalization error for adaptive margin classifiers \citep{zhu2023principled, li2023towards, Bousquet2004IntroductionTS, shalev2014understanding}, and the bound therefore provides intuition for the empirical findings rather than direct guarantees. Additionally, computing the margin in DPO-PoP requires one additional forward pass through the target policy per training step, incurring a modest increase in compute relative to standard DPO.

\section{Large Language Model Usage Statement}\label{ap:llm-usage}

Large Language Models (LLMs) were used for grammatical editing and improving writing flow. In addition, they were used to verify the correctness of our theoretical proof for the generalization bounds on performance of adaptive-margin alignment algorithms. They were also used to gather PoP annotations in our LLM annotated PoP experiments. The research methodology, experimental design, data analysis, and all scientific conclusions are entirely the work of the human authors.

\section{Experiment Details}\label{ap:exp_details}

The hyperparameters used in our experiments for SFT and DPO are provided in Table \ref{tab:hyperparams-sft} and Table \ref{tab:hyperparams-dpo} respectively. For DPO-PoP, we used the same hyperparameters used for DPO. For the DPO-PoP specific hyperparameters we set the clipping threshold $M_\text{max}=10$, the Polyak averaging decay rate to $\tau=0.99$, and the size of the PoP dataset to $120,000$ (twice the size of the preference dataset in UltraFeedback, i.e $k=2$). All models were trained using 4 Nvidia A100 80G GPUs. We also ablate the compute budget: since DPO-PoP ($k=2$) trains on twice the data as the baselines, we trained baselines for 2 epochs to match this budget. However, performance generally degraded with 2 epochs of training, motivating our choice of 1 epoch for all baselines. Full results are in Appendix~\ref{ap:compute-parity}. Our codebase is built on top of the Safe-RLHF repository \citep{dai2023saferlhfsafereinforcement} (\url{https://github.com/pku-alignment/safe-rlhf}). Our modified code base is available at $\underline{\textit{removed for review}}$
% \textcolor{red}{[ADD THE GITHUB LINK HERE]}

\begin{table}[h]
\centering
\begin{tabular}{lc}
\toprule
\textbf{Hyperparameter} & \textbf{Value} \\
\midrule
Optimizer & AdamW \\
Epochs & 1 \\
Max Sequence Length & 2048 \\
Per-device Train Batch Size & 2 \\
Per-device Eval Batch Size & 2 \\
Gradient Accumulation Steps & 8 \\
Gradient Checkpointing & True \\
Num GPUs & 4 \\
Learning Rate & 2e-5 \\
Learning Rate Scheduler & Cosine \\
Weight Decay & 0 \\
\bottomrule
\end{tabular}
\caption{Training hyperparameters used for SFT}
\label{tab:hyperparams-sft}
\end{table}

\begin{table}[h]
\centering
\begin{tabular}{lc}
\toprule
\textbf{Hyperparameter} & \textbf{Value} \\
\midrule
Optimizer & AdamW \\
Epochs & 1 \\
Max Sequence Length & 2048 \\
Per-device Train Batch Size & 2 \\
Per-device Eval Batch Size & 2 \\
Gradient Accumulation Steps & 8 \\
Gradient Checkpointing & True \\
Num GPUs & 4 \\
Learning Rate & 1e-6 \\
Learning Rate Scheduler & Cosine \\
Learning Rate Warmup Ratio & 0.03 \\
Weight Decay & 0.05 \\
Beta & 0.1 \\
\bottomrule
\end{tabular}
\caption{Training hyperparameters used for DPO}
\label{tab:hyperparams-dpo}
\end{table}

\section{Reinforcement Learning from Human Feedback}\label{ap:rlhf}

Reinforcement Learning from Human Feedback (RLHF) \citep{ouyang2022training} is the predominant paradigm for aligning language models with human intent. The RLHF pipeline typically begins with a pre-trained language model trained on an internet-scale corpus and proceeds through three stages. We briefly describe each stage below:

\textbf{Supervised Fine Tuning} In the SFT stage, the model is fine-tuned to follow instructions by autoregressively predicting the next token in a sequence using Maximum Likelihood Estimation (MLE). This stage uses a dataset $\mathcal{D}_{\text{SFT}}$ consisting of prompt-response pairs $(x, y)$, where $x$ is a prompt and $y$ is a high-quality response. These responses are either human-annotated or generated by large language models.

\textbf{Reward Modeling} In the reward modeling stage, a reward model is trained to assign scalar scores to prompt-response pairs, indicating how well a response aligns with human preferences. This process relies on a preference dataset $\mathcal{D}_{\text{pref}} = {(x_i, y^+_i, y^-_i)}_{i=1}^{N}$, where $x_i$ is a prompt, $y^+_i$ is the preferred response, and $y^-_i$ is the dispreferred response. Preference labels are typically provided by human annotators or large language models. The Bradley-Terry (BT) model \citep{bradley1952rank} is commonly used to model the likelihood of observed preferences.

\begin{equation}
    P(y^{+} \succ y^{-}) = \frac{e^{r(x,y^{+})}}{e^{r(x,y^{+})}+e^{r(x,y^{-})}} = \sigma(r(x,y^{+})-r(x,y^{-}))
\end{equation}

Here, $r$ denotes the 
%\sn{again, I think you should get rid of latent and just say reward}
reward assigned to a prompt-response pair, and $\sigma$ denotes the logistic (sigmoid) function. We parameterize the reward function as $r_\phi$, where $\phi$ represents the model parameters, and use it to approximate the ground-truth reward function. The reward model is trained by maximizing the likelihood of the observed preference data under the Bradley-Terry model.

\begin{equation}
    \min_{\phi} -\mathbb{E}_{(x,y^{+},y^{-}) \sim \mathcal{D}_\text{pref}}[\log \sigma(r_{\phi}(x,y^{+})-r_{\phi}(x,y^{-}))]
\end{equation}

\textbf{Reinforcement Learning} In the reinforcement learning stage, the language model is optimized to generate responses that maximize the reward assigned by the learned reward model $r_{\phi}$. However, directly optimizing for this reward can degrade response quality, as the policy may overfit to imperfections in the learned reward function and begin producing 
%adversarial \sn{is adversarial the right word here?} or 
unnatural outputs \citep{Jaques2019WayOB, stiennon2022learningsummarizehumanfeedback}.

To mitigate this, a KL divergence constraint is added to ensure that the updated policy does not deviate too far from a reference policy, usually taken to be the supervised fine-tuning (SFT) policy. The resulting RL objective, with a KL penalty coefficient $\beta$, is given by:

\begin{equation}\label{eq:rl}
    \max_{\theta} \mathbb{E}_{x\sim\mathcal{D},y\sim\pi_{\theta}(.\vert x)}[r_{\phi}(x,y)] - \beta \mathbb{D}_{KL}[\pi_{\theta}(y\vert x)\vert\vert \pi_{ref}(y\vert x)]
\end{equation}

Additionally, some approaches \citep{chittepu2025reinforcementlearninghumanfeedback, dai2023saferlhfsafereinforcement} enforce safety and harmlessness by augmenting the objective in Equation~\ref{eq:rl} with an explicit cost constraint.

\section{Results for Llama-3.1-8b}\label{ap:llama3.1-8b}

\subsection{Discriminative Performance}

The results showing the test classification accuracy on the UltraFeedback dataset \citep{cui2024ultrafeedbackboostinglanguagemodels} and RewardBench \citep{lambert2024rewardbenchevaluatingrewardmodels} scores are in Tables \ref{tab:dpo-margin-results-llama3.1-8b} and \ref{tab:dpo-rewardbench-llama3.1-8b} respectively. 

\begin{table}[h]
\centering
\caption{Comparison of DPO variants on classification accuracy and Spearman, Pearson correlation with ground-truth margins for Llama-3.1-8b, when using synthetic data.}
\label{tab:dpo-margin-results-llama3.1-8b}
\begin{tabular}{lccc}
\toprule
\textbf{Algorithm} & \textbf{Pearson Correlation} & \textbf{Spearman Correlation} & \textbf{Accuracy\%} \\
\midrule
Vanilla DPO & 0.3151 & 0.3244 & 69.0 \\
DPO-margin-1 & 0.3161 & 0.3243 & 69.0 \\
DPO-margin-gt &  0.3791 & 0.3715 & 70.0 \\
DPO-margin-gt-scaled & 0.3633 & 0.3669 & 71.0 \\
DPO-PoP-iter & 0.2183 & \underline{\textbf{0.3868}} & \underline{\textbf{82.0}} \\
DPO-PoP-random & \underline{\textbf{0.3962}} & \underline{\textbf{0.3871}} & 71.0 \\
\bottomrule
\end{tabular}
\end{table}

\begin{table}[h]
\centering
\caption{Performance of Llama-3.1-8b DPO variants on RewardBench, when using synthetic data. Higher is better.}
\label{tab:dpo-rewardbench-llama3.1-8b}
\begin{tabular}{lccccc}
\toprule
\textbf{Model} & \textbf{Chat} & \textbf{Chat Hard} & \textbf{Safety} & \textbf{Reasoning} & \textbf{Overall} \\
\midrule
Vanilla-DPO & 73.46 & 63.60 & 57.03 & 76.69 & 71.59 \\
DPO-margin-1 & 71.23 & 62.94 & 57.16 & \underline{\textbf{77.07}} & 71.39 \\
DPO-margin-gt & 79.05 & 65.79 & 60.95 & 76.84 & 73.67 \\
DPO-margin-gt-scaled & 76.26 & 62.28 & 62.43 & 76.11 & 72.96 \\
DPO-PoP-iter & \underline{\textbf{86.59}} & 61.84 & \underline{\textbf{72.03}} & 72.05 & 75.41 \\
DPO-PoP-random & 81.56 & \underline{\textbf{66.89}} & 68.51 & 76.95 & \underline{\textbf{76.25}} \\
\bottomrule
\end{tabular}
\end{table}

\subsection{Generative Performance}

The results displaying the win rate of the model responses as judged by UltraRM \citep{cui2024ultrafeedbackboostinglanguagemodels} and AlpacaEval 2.0 win rates \citep{dubois2025lengthcontrolledalpacaevalsimpleway} are in Tables \ref{tab:dpo-gen-results-llama3.1-8b} and \ref{tab:dpo-alpaca-llama3.1-8b} respectively.

\begin{table}[H]
\centering
\caption{Comparison of margin-based DPO variants against Vanilla DPO on median advantage and win rate for Llama-3.1-8b, when using synthetic data.}
\label{tab:dpo-gen-results-llama3.1-8b}
\begin{tabular}{lcc}
\toprule
\textbf{Method} & \textbf{Median Advantage} & \textbf{Win Rate \%} \\
\midrule
DPO-margin-1 & 0.2813 & 55\% \\
DPO-margin-gt & 0.5000 & 59\% \\
DPO-margin-gt-scaled & 0.0938 & 52\% \\
DPO-PoP-iter & 0.3496 & 56\% \\
DPO-PoP-random & \underline{\textbf{0.7500}} & \underline{\textbf{63\%}} \\
\bottomrule
\end{tabular}
\end{table}

\begin{table}[H]
\centering
\caption{Performance of Llama-3.1-8b DPO variants on the AlpacaEval 2.0 benchmark, when using synthetic data.}
\label{tab:dpo-alpaca-llama3.1-8b}
\begin{tabular}{lccc}
\toprule
\textbf{Experiment} & \textbf{Length-Controlled Win Rate \%} & \textbf{Win Rate \%} & \textbf{Avg Length} \\
\midrule
Vanilla-DPO & 10.38 & 10.56 & 1869 \\
DPO-margin-1 & 11.07 & 11.06 & 1864 \\
DPO-margin-gt & 11.23 & 11.30 & 1825 \\
DPO-margin-gt-scaled & 10.95 & 11.43 & 1881 \\
DPO-PoP-iter & 12.89 & 13.42 & 2004 \\
DPO-PoP-random & \underline{\textbf{14.62}} & \underline{\textbf{14.78}} & 1909 \\
\bottomrule
\end{tabular}
\end{table}

\section{Compute-Matched Baseline Evaluation}\label{ap:compute-parity}

Since DPO-PoP with $k=2$ trains on a dataset twice the size of the preference dataset, its effective training budget is doubled relative to baselines trained for a single epoch. To verify that the improvements of DPO-PoP-Random are not solely attributable to this compute advantage, we trained all DPO baselines for 2 epochs on the full UltraFeedback preference dataset. Table~\ref{tab:compute-parity} reports mean UltraRM win rates, RewardBench Overall scores, and test classification accuracy for each baseline at 1 and 2 epochs, alongside both DPO-PoP variants trained for a single epoch.

{\newcommand{\inc}[1]{\textcolor{green!50!black}{$\uparrow$}}
\newcommand{\dec}[1]{\textcolor{red}{$\downarrow$}}
\begin{table*}[h]
\centering
\footnotesize
\caption{UltraRM win rate, RewardBench Overall, and test classification accuracy for DPO baselines at 1 vs.\ 2 epochs, compared to DPO-PoP variants at 1 epoch (Llama-3.2-3b, synthetic data). Colored deltas show change relative to 1-epoch performance.}
\label{tab:compute-parity}
\setlength{\tabcolsep}{4pt}
\begin{tabular}{lcccccc}
\toprule
\textbf{Algorithm} & \multicolumn{2}{c}{\textbf{UltraRM Win Rate\%}} & \multicolumn{2}{c}{\textbf{RewardBench Overall}} & \multicolumn{2}{c}{\textbf{Test Accuracy\%}} \\
\cmidrule(lr){2-3}\cmidrule(lr){4-5}\cmidrule(lr){6-7}
 & \textbf{1 Epoch} & \textbf{2 Epochs} & \textbf{1 Epoch} & \textbf{2 Epochs} & \textbf{1 Epoch} & \textbf{2 Epochs} \\
\midrule
Vanilla DPO          & - & - & 75.46 & 73.33~\dec{2.1} & 71.15 & 68.36~\dec{2.8} \\
DPO-margin-1         & 54.9 & 57.6~\inc{2.7} & 75.78 & 73.53~\dec{2.3} & 71.18 & 68.66~\dec{2.5} \\
DPO-margin-gt        & 61.3 & 56.1~\dec{5.2} & 77.45 & 75.75~\dec{1.7} & 71.85 & 69.67~\dec{2.2} \\
DPO-margin-gt-scaled & 53.7 & 59.5~\inc{5.8} & 77.13 & 75.21~\dec{1.9} & 72.05 & 70.03~\dec{2.0} \\
\midrule
DPO-PoP-iter         & 57.8          & -- & 76.73          & -- & \textbf{79.97} & -- \\
DPO-PoP-random       & \textbf{62.4} & -- & \textbf{78.87} & -- & 71.09          & -- \\
\bottomrule
\end{tabular}
\end{table*}}

Training baselines for 2 epochs does not consistently improve performance. RewardBench Overall scores and test accuracy degrade for all baselines. UltraRM win rate is mixed, it decreases for DPO-margin-gt but increases slightly for DPO-margin-1 and DPO-margin-gt-scaled. Moreover, the best 2-epoch baseline UltraRM win rate (59.5\%) is lower than the best 1-epoch baseline (61.3\%), and neither reaches DPO-PoP-Random's 62.4\%. Taken together, these results demonstrate that the gains from DPO-PoP-Random are attributable to the PoP objective rather than the increased training budget.

\section{Effect of PoP data scale on performance}\label{ap:pop-data-scale}

In order to study the effect of the PoP data scale on model performance, we consider the Llama-3.2-3B model and begin with an initial subset of preferences of size \( |\mathcal{D}_{\text{pref}}| = 7500 \). We then generate a Preference-over-Preference (PoP) dataset of size \( k \cdot |\mathcal{D}_{\text{pref}}| \), where \( k \in \{1, 2, 4, 8, 16\} \). This procedure is carried out using both iterative and random sampling strategies for generating the PoP data. The baseline DPO variants are all trained on the same subset of 7500 preferences used to construct the PoP dataset.

\subsection{Discriminative Performance}
\begin{table}[h!]
\centering
\caption{Comparison of baseline DPO variants trained on a subset of preferences ($\vert \mathcal{D}_{\text{pref}}\vert = 7500$), evaluated on classification accuracy and correlation with ground-truth margins for Llama-3.2-3b.}
\label{tab:dpo-k-disc-baselines}
\begin{tabular}{lccc}
\toprule
\textbf{Algorithm} & \textbf{Pearson Correlation} & \textbf{Spearman's Correlation} & \textbf{Accuracy\%} \\
\midrule
Vanilla-DPO & 0.1450 & 0.1708 & 64.0 \\
DPO-margin-1 & 0.1374 & 0.1609 & 64.0 \\
DPO-margin-gt & 0.1855 & 0.2091 & 65.0 \\
DPO-margin-gt-scaled & 0.1441 & 0.1656 & 64.0 \\
\bottomrule
\end{tabular}
\end{table}

\begin{table}[h!]
\centering
\caption{Performance of DPO-PoP-iter for varying values of $k$, evaluated on classification accuracy and correlation with ground-truth margins for Llama-3.2-3b.}
\label{tab:dpo-k-disc-pop-iter}
\begin{tabular}{lccc}
\toprule
\textbf{Data Size Multiplier $k$} & \textbf{Pearson Correlation} & \textbf{Spearman's Correlation} & \textbf{Accuracy\%} \\
\midrule
1 & \underline{\textbf{0.2229}} & \underline{\textbf{0.2463}} & \underline{\textbf{67.0}} \\
2 & 0.2193 & 0.2429 & \underline{\textbf{67.0}} \\
4 & 0.2127 & 0.2325 & 65.0 \\
8 & 0.2183 & 0.2268 & 64.0 \\
16 & 0.2223 & 0.2236 & 63.0 \\
\bottomrule
\end{tabular}
\end{table}

\begin{table}[h!]
\centering
\caption{Performance of DPO-PoP-random for varying values of $k$, evaluated on classification accuracy and correlation with ground-truth margins for Llama-3.2-3b.}
\label{tab:dpo-k-disc-pop-random}
\begin{tabular}{lccc}
\toprule
\textbf{Data Size Multiplier $k$} & \textbf{Pearson Correlation} & \textbf{Spearman's Correlation} & \textbf{Accuracy\%} \\
\midrule
1 & 0.2386 & 0.2614 & \underline{\textbf{67.0}} \\
2 & \underline{\textbf{0.2403}} & \underline{\textbf{0.2638}} & 66.0 \\
4 & 0.2362 & 0.2556 & 66.0 \\
8 & 0.2322 & 0.2454 & 65.0 \\
16 & 0.2265 & 0.2354 & 66.0 \\
\bottomrule
\end{tabular}
\end{table}

Comparing Table~\ref{tab:dpo-k-disc-baselines} with Tables~\ref{tab:dpo-k-disc-pop-iter} and~\ref{tab:dpo-k-disc-pop-random}, we observe that the DPO-PoP variants consistently outperform the DPO baselines in terms of discriminative performance, including those baselines that have access to ground-truth margins. Furthermore, increasing the data size multiplier $k$ results in a decline in classification accuracy and correlation metrics with respect to the ground-truth margins for both DPO-PoP variants. Notably, this performance degradation is more pronounced in DPO-PoP-iter than in DPO-PoP-random. These findings suggest that, when prioritizing discriminative performance, using smaller values of $k$ (e.g., $k=1$ or $k=2$) is advisable.

\subsection{Generative Performance}

\begin{table}[h!]
\centering
\caption{Median advantage and win rate of various DPO baseline variants over Vanilla-DPO, for Llama-3.2-3b. All models are trained on a subset of preferences with $|\mathcal{D}_{\text{pref}}| = 7500$.}
\label{tab:dpo-k-gen-baselines}
\begin{tabular}{lcc}
\toprule
\textbf{Method} & \textbf{Median Advantage} & \textbf{Win Rate\%} \\
\midrule
DPO-margin-1 & 0.2500 & 56.0 \\
DPO-margin-gt & 0.4844 & 60.0 \\
DPO-margin-gt-scaled & 0.0313 & 51.0 \\
\bottomrule
\end{tabular}
\end{table}

\begin{table}[h!]
\centering
\caption{Median advantage and win rate of DPO-PoP-iter over Vanilla-DPO for different values of $k$, for Llama-3.2-3b.}
\label{tab:dpo-k-gen-pop-iter}
\begin{tabular}{lcc}
\toprule
\textbf{Data Size Multiplier $k$} & \textbf{Median Advantage} & \textbf{Win Rate\%} \\
\midrule
1 & 0.2813 & 55.0 \\
2 & 1.1250 & 68.0 \\
4 & \underline{\textbf{1.7813}} & \underline{\textbf{77.0}} \\
8 & 1.7188 & 75.0 \\
16 & 1.4629 & 69.0 \\
\bottomrule
\end{tabular}
\end{table}

\begin{table}[h!]
\centering
\caption{Median advantage and win rate of DPO-PoP-random over Vanilla-DPO for different values of $k$, for Llama-3.2-3b.}
\label{tab:dpo-k-gen-pop-random}

\begin{tabular}{lcc}
\toprule
\textbf{Data Size Multiplier $k$} & \textbf{Median Advantage} & \textbf{Win Rate\%} \\
\midrule
1 & 0.4688 & 57.0 \\
2 & 1.2500 & 71.0 \\
4 & 1.7969 & \underline{\textbf{77.0}} \\
8 & \underline{\textbf{1.8711}} & \underline{\textbf{77.0}} \\
16 & 1.5547 & 72.0 \\
\bottomrule
\end{tabular}
\end{table}

Looking at Tables~\ref{tab:dpo-k-gen-pop-iter} and~\ref{tab:dpo-k-gen-pop-random}, we observe that the win rate initially increases with the data size multiplier $k$, before eventually declining. We hypothesize that this due to the fact, at small $k$, the model has insufficient PoP supervision to learn accurate margin estimates, while at large $k$, it overfits to the margins inferred from the PoP training data. Additionally, DPO-PoP-random appears to be more robust to the choice of $k$ than DPO-PoP-iter when considering win rate. When prioritizing generative ability, a moderately larger value of $k$ (e.g., $k=4$ or $k=8$) is preferable. More importantly, when comparing with Table~\ref{tab:dpo-k-gen-baselines}, we find that in a small-data regime, DPO-PoP variants achieve substantially higher win rates than the DPO baselines—including those with access to ground-truth margins.

\section{Effect of PoP labeling noise on performance}\label{ap:pop-noise}

\subsection{Random Label Flipping}
We investigate the sensitivity of our DPO-PoP approaches to noise in PoP labels collected from annotators. Given our PoP dataset $\vert \mathcal{D}_\text{PoP} \vert$, we introduce label noise by randomly flipping PoP labels with probability $\epsilon$. We use the Llama-3.2-3b model and experiment with three different noise levels: $\epsilon \in \{0.1, 0.3, 0.5\}$. We evaluate both the discriminative and generative performance of models trained on these perturbed datasets.

\subsubsection{Discriminative Performance}

We observe from Figure~\ref{fig:pop-noise-disc} that both the Spearman and Pearson correlations for DPO-PoP-iter and DPO-PoP-random decrease as the noise level increases. Notably, this decline in correlation is more pronounced for DPO-PoP-iter compared to DPO-PoP-random. From the accuracy plot, we surprisingly find that the test classification accuracy of DPO-PoP-iter slightly increases with added noise, while it marginally decreases for DPO-PoP-random. We hypothesize that label noise induces a regularizing effect in DPO-PoP-iter, which helps mitigate its tendency to overfit to weaker preferences.

\begin{figure}[htbp]
    \centering
    \includegraphics[width=1.0\textwidth]{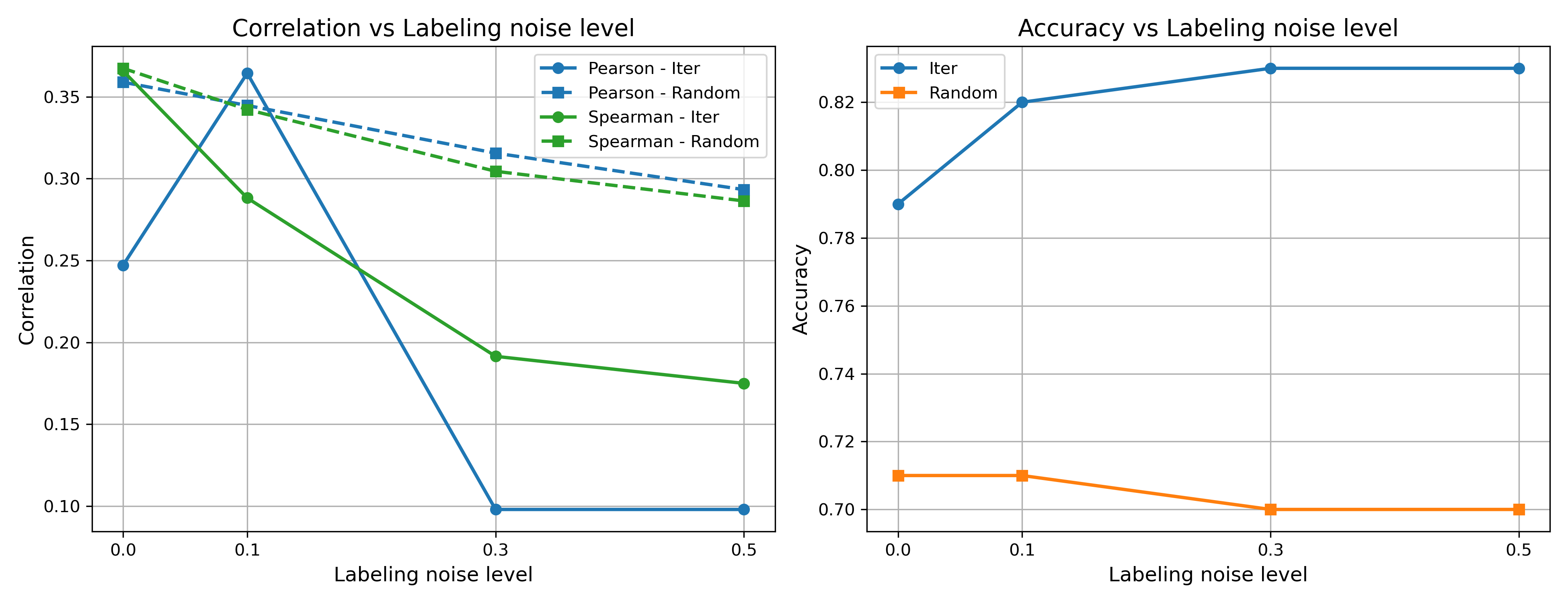}
    \caption{Spearman and Pearson correlations (left), and test classification accuracy (right) of DPO-PoP models trained with varying levels of label noise.}
    \label{fig:pop-noise-disc}
\end{figure}

\subsubsection{Generative Performance}

We observe from Figure~\ref{fig:pop-noise-gen} that both the win rate and median advantage for DPO-PoP-random decrease as the noise level increases. The win rate and median advantage for DPO-PoP-Iter also display a declining trend as noise increases.

% Similar to the trend observed in the discriminative setting, we find that the win rate and median advantage for DPO-PoP-iter initially increase before declining, suggesting that a moderate amount of label noise may have a regularizing effect, helping DPO-PoP-iter avoid overfitting to weaker preferences.

\begin{figure}[htbp]
    \centering
    \includegraphics[width=1.0\textwidth]{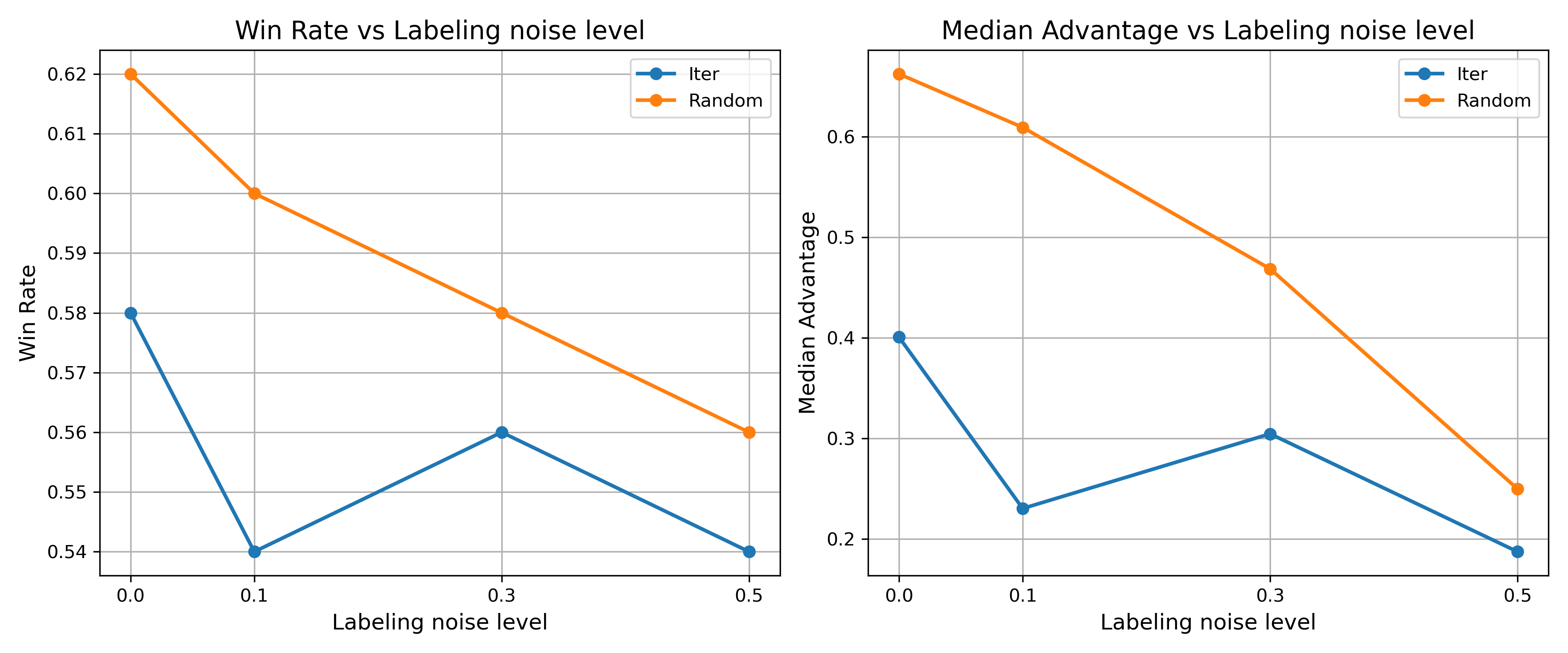}
    \caption{Win rates (left) and median advantage (right) of DPO-PoP models trained with varying levels of label noise.}
    \label{fig:pop-noise-gen}
\end{figure}

\subsection{Margin-Dependent Label Flipping}

In contrast to uniform noise, we also introduce structured label corruption calibrated to annotation difficulty. Specifically, we modulate the temperature parameter of the Bradley-Terry model to simulate annotator uncertainty: pairs with similar preference strengths, and thus smaller margin differences, are flipped with higher probability, while more discrepant preference pairs are largely unaffected. Given two preference instances $D_{s} = \{x_s, y^{+}_{s}, y^{-}_{s}\}$ and 
$D_{w} = \{x_w, y^{+}_{w}, y^{-}_{w}\}$, let $m(x, y^{+}, y^{-})$ denote a margin 
function that computes the reward gap between the preferred and dispreferred responses. 
The probability of retaining the correct PoP label for a given pair under temperature $\tau$ is
then defined following Equation~\ref{eq:bt_flip}. As $\tau$ increases, PoP labels for preference pairs with similar strengths (small margin differences) are more likely to be flipped, while those for pairs with large margin differences remain largely unaffected. This reflects a more realistic noise model, where annotator disagreement is more likely on ambiguous preference pairs. Results from  Tables ~\ref{tab:uf-test-acc-bt-noise} and ~\ref{tab:ultrarm-winrate-bt-noise} demonstrate that PoP methods are robust to annotator noise.

\begin{equation}
    P(D_{s} \succ D_{w}) = \sigma\!\left(\frac{m(x_s, y^{+}_{s}, y^{-}_{s}) - m(x_w, y^{+}_{w}, y^{-}_{w})}{\tau}\right)
    \label{eq:bt_flip}
\end{equation}

\subsubsection{Discriminative Performance}

Table~\ref{tab:uf-test-acc-bt-noise} reports test classification accuracy across temperatures. Both \texttt{DPO-PoP-Iter} and \texttt{DPO-PoP-Random} remain stable or show a slight upward trend with increasing temperature. Rather than degrading, performance is maintained or marginally improved, suggesting that flipping weaker PoP labels more frequently introduces a regularizing effect that benefits classification.

\begin{table}[h]
\centering
\caption{Test Classification Accuracy under Margin-Dependent Label Noise.}
\label{tab:uf-test-acc-bt-noise}
\begin{tabular}{lccccc}
\toprule
\textbf{Method} & \textbf{Temp-1} & \textbf{Temp-2} & \textbf{Temp-4} & \textbf{Temp-8} & \textbf{Temp-16} \\
\midrule
DPO-PoP-Iter   & 80.60 & 82.22 & 83.43 & 83.27 & 84.28 \\
DPO-PoP-Random & 67.20 & 68.41 & 69.27 & 68.26 & 69.52 \\
\bottomrule
\end{tabular}
\end{table}

\subsubsection{Generative Performance}

Table~\ref{tab:ultrarm-winrate-bt-noise} reports UltraRM win rates against Vanilla DPO. \texttt{DPO-PoP-Iter} remains stable with increasing temperatures, while \texttt{DPO-PoP-Random} trends slightly upward at higher temperatures, consistent with the regularizing effect.

\begin{table}[h]
\centering
\caption{UltraRM Win Rates under Margin-Dependent Label Noise.}
\label{tab:ultrarm-winrate-bt-noise}
\begin{tabular}{lccccc}
\toprule
\textbf{Method} & \textbf{Temp-1} & \textbf{Temp-2} & \textbf{Temp-4} & \textbf{Temp-8} & \textbf{Temp-16} \\
\midrule
DPO-PoP-Iter   & 0.56 & 0.54 & 0.56 & 0.55 & 0.56 \\
DPO-PoP-Random & 0.56 & 0.56 & 0.56 & 0.57 & 0.59 \\
\bottomrule
\end{tabular}
\end{table}

As we increase the temperature, both discriminative and generative performance of DPO-PoP variants remain stable or marginally improve, demonstrating robustness to margin-dependent annotation noise.

\section{Evolution of Metrics over Training}

In this section, we present the evolution of test classification accuracy, the KL divergence with respect to the reference policy, and the Ultra-RM win rate over the course of training, in Figure ~\ref{fig:evolution-plots}. Note that for the PoP methods, because we use $k=2$, the effective training budget is doubled; this is due to the training dataset being twice the size of the original preference dataset. The plots are averaged over 5 seeds. We point that for the KL and test classification accuracy plots, the confidence intervals are very small, which is why they are not visible in the plots.

\begin{figure}[t]
    \centering

    % ---------- Top row ----------
    \begin{subfigure}{0.48\textwidth}
        \centering
        \includegraphics[width=\textwidth]{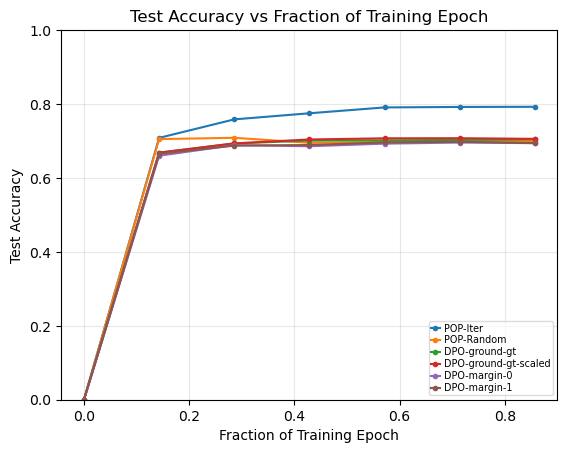}
        \caption{Test Accuracy}
    \end{subfigure}
    \hfill
    \begin{subfigure}{0.48\textwidth}
        \centering
        \includegraphics[width=\textwidth]{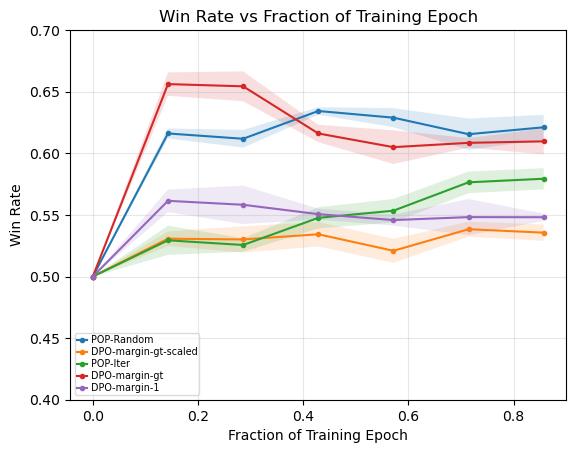}
        \caption{UltraRM-Winrate}
    \end{subfigure}

    \vspace{0.5em}

    % ---------- Bottom row ----------
    \begin{subfigure}{0.48\textwidth}
        \centering
        \includegraphics[width=\textwidth]{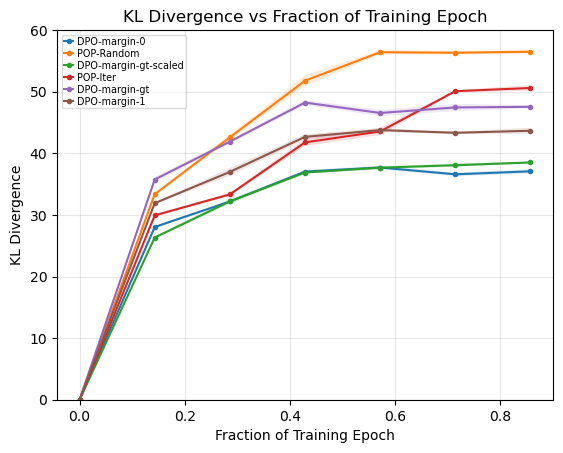}
        \caption{KL Divergence}
    \end{subfigure}

    \caption{Training curves for test classification accuracy, UltraRM-winrate, and KL with respect to the reference policy.}
    \label{fig:evolution-plots}
\end{figure}

% \begin{figure}[t]
%     \centering
%     \begin{subfigure}{0.31\textwidth}
%         \centering
%         \includegraphics[width=1.1\textwidth]{iclr2026/figs/all_experiments_test_accuracy.png}
%         \caption{Test Accuracy}
%     \end{subfigure}
%     % \hfill
%     \begin{subfigure}{0.31\textwidth}
%         \centering
%         \includegraphics[width=1.1\textwidth]{iclr2026/figs/all_experiments_win_rate.png}
%         \caption{UltraRM-Winrate}
%     \end{subfigure}
%     % \hfill
%     \begin{subfigure}{0.31\textwidth}
%         \centering
%         \includegraphics[width=1.1\textwidth]{iclr2026/figs/all_experiments_kl.png}
%         \caption{KL Divergence}
%     \end{subfigure}
%     \caption{Training curves for test classification accuracy, UltraRM-winrate, and KL with respect to reference policy.}
%     \label{fig:evolution-plots}
% \end{figure}

% \begin{figure}[t]
%     \centering
%     \subfigure[Test Accuracy]{
%         \includegraphics[width=0.31\textwidth]{iclr2026/figs/all_experiments_test_accuracy.png}}
%     \hfill
%     \subfigure[UltraRM Winrate]{
%         \includegraphics[width=0.31\textwidth]{iclr2026/figs/all_experiments_win_rate.png}}
%     \hfill
%     \subfigure[KL Divergence]{
%         \includegraphics[width=0.31\textwidth]{iclr2026/figs/all_experiments_kl.png}}
%     \caption{Training curves for test classification accuracy, UltraRM-winrate, and KL with respect to reference policy.}
%     \label{fig:evolution-plots}
% \end{figure}

\section{Bounds on the Generalization performance of Adaptive margin classifiers }
\label{ap:theorem}
Here, we analyze the generalization performance of adaptive margin classifiers from a theoretical perspective. We restrict ourselves to reward model inference from preferences. Furthermore, we assume linear reward functions. The reward difference between chosen and rejected responses in a preference pair $(x,y^+,y^-)$ can be expressed as $g_{w}(\psi) = r(x,y^+)-r(x,y^-) = w^T\psi(x,y^+,y^-)$. 

\subsection{Setting}

Let $(\Psi,M)$ be a random pair with distribution $\mathcal{D}$, where
\[
  \Psi \in \mathbb{R}^d,\qquad M \in (0,\infty).
\]

Here $\Psi$ and $M$ are random variables corresponding to feature differences and margins respectively. 
We observe an i.i.d.\ sample
\[
  S = \{(\psi_i,m_i)\}_{i=1}^n \sim \mathcal{D}^n.
\]
Assume
\begin{equation}
  \|\psi_i\|_2 \le R \quad \text{for all } i=1,\dots,n,
  \label{eq:radius-assumption}
\end{equation}
for some $R>0$.
We consider linear predictors $w \in \mathbb{R}^d$ with
\begin{equation}
  \|w\|_2 \le \Lambda,
  \label{eq:norm-assumption}
\end{equation}
for some $\Lambda>0$.
For $w$ and a data point $(\psi,m)$ we define the score
\[
  g_w(\psi) := w^\top \psi.
\]

The \emph{test misclassification error} of $w$ (with no access to $M$ at test time) is
\begin{equation}
  L(w) := \Pr_{(\Psi,M)\sim\mathcal{D}}\!\bigl( g_w(\Psi) \le 0 \bigr).
  \label{eq:test-error}
\end{equation}

For each training point $i$, define
\[
  g_i(w) := g_w(\psi_i) = w^\top \psi_i.
\]

\paragraph{Adaptive-margin logistic loss.}
Given a per-example margin $m_i>0$, define the \emph{shifted logistic loss}
\begin{equation}
  \ell_i(w) := \log\bigl(1 + \exp\bigl( -(g_i(w) - m_i) \bigr)\bigr).
  \label{eq:logistic-loss}
\end{equation}
The empirical adaptive-margin logistic loss is
\begin{equation}
  \hat{L}_{\log}(w)
  := \frac{1}{n}\sum_{i=1}^n \ell_i(w)
  = \frac{1}{n}\sum_{i=1}^n
  \log\bigl(1 + \exp\bigl( -(w^\top\psi_i - m_i) \bigr)\bigr).
  \label{eq:empirical-logistic-loss}
\end{equation}

\paragraph{Ramp loss with per-example margin.}
For $m>0$ define the (margin-$m$) ramp loss
\begin{equation}
  \Phi_m(u) :=
  \begin{cases}
    1, & u \le 0,\\[4pt]
    1 - \dfrac{u}{m}, & 0 < u < m,\\[4pt]
    0, & u \ge m.
  \end{cases}
  \label{eq:ramp-loss}
\end{equation}
Note that $0 \le \Phi_m(u) \le 1$ for all $u$ and $m$, and that
\begin{equation}
  \mathbf{1}\{u \le 0\} \;\le\; \Phi_m(u)
  \quad \text{for all } u\in\mathbb{R}, \; m>0.
  \label{eq:01-vs-ramp}
\end{equation}

\subsection{Main Theorem}

We now state the desired generalization bound, in which the empirical term is exactly (up to a universal constant factor) the empirical adaptive-margin logistic loss~\eqref{eq:empirical-logistic-loss}.
\\
\begin{theorem}[Adaptive-margin logistic generalization bound]
\label{thm:main}
Assume~\eqref{eq:radius-assumption} and~\eqref{eq:norm-assumption}, and let $\delta \in (0,1)$.
Then with probability at least $1-\delta$ over the sample $S\sim\mathcal{D}^n$, we have simultaneously for all $w$ with $\|w\|_2 \le \Lambda$,
\begin{equation}
  \Pr_{(\Psi,M)\sim\mathcal{D}}\!\bigl( w^\top \Psi \le 0 \bigr)
  \;\le\;
  \frac{1}{\log 2}\,\hat{L}_{\log}(w)
  \;+\;
  \frac{2\Lambda R}{n}\,\sqrt{\sum_{i=1}^n \frac{1}{m_i^2}}
  \;+\;
  \sqrt{\frac{2\log(2/\delta)}{n}}.
  \label{eq:main-bound}
\end{equation}
In particular, the left-hand side depends only on the test score $w^\top \Psi$ and does not require access to $M$ at test time; the adaptive margins $m_i$ appear only in the empirical loss and in the margin-distribution complexity term.
\end{theorem}

The rest of this note is devoted to the proof.

\subsection{From 0--1 loss to ramp loss}

We first express the test error~\eqref{eq:test-error} in terms of the ramp loss~\eqref{eq:ramp-loss}.

\begin{lemma}
\label{lem:01-to-ramp}
For any $w\in\mathbb{R}^d$,
\begin{equation}
  L(w)
  \;=\;
  \mathbb{E}_{(\Psi,M)\sim\mathcal{D}}
  \bigl[\mathbf{1}\{w^\top\Psi \le 0\}\bigr]
  \;\le\;
  \mathbb{E}_{(\Psi,M)\sim\mathcal{D}}
  \bigl[\Phi_M(w^\top\Psi)\bigr].
  \label{eq:01-to-ramp}
\end{equation}
\end{lemma}

\begin{proof}
For any fixed $(\psi,m)$ and $w$ we have~\eqref{eq:01-vs-ramp}:
\[
  \mathbf{1}\{w^\top\psi \le 0\} \;\le\; \Phi_m(w^\top\psi).
\]
Taking expectation over $(\Psi,M)\sim\mathcal{D}$ yields
\[
  \mathbb{E}\bigl[\mathbf{1}\{w^\top\Psi \le 0\}\bigr]
  \;\le\;
  \mathbb{E}\bigl[\Phi_M(w^\top\Psi)\bigr].
\]
The left-hand side is $L(w)$ by~\eqref{eq:test-error}, giving~\eqref{eq:01-to-ramp}.
\end{proof}

Thus it suffices to obtain a uniform upper bound on
\[
  \mathbb{E}\bigl[\Phi_M(w^\top\Psi)\bigr]
\]
in terms of the empirical ramp loss
\[
  \frac{1}{n}\sum_{i=1}^n \Phi_{m_i}\bigl(g_i(w)\bigr)
\]
and a complexity term.

\subsection{Uniform bound for the ramp loss}

Define the function class
\[
  \mathcal{H}
  :=
  \bigl\{
    h_w : (\psi,m)\mapsto \Phi_m(w^\top\psi)
    \,\big|\,
    \|w\|_2 \le \Lambda
  \bigr\}.
\]
Each $h_w$ maps into $[0,1]$.
We use the standard Rademacher-complexity generalization bound for bounded losses.

\begin{lemma}[Uniform deviation for bounded losses]
\label{lem:rad-ub}
Let $\mathcal{H} \subseteq [0,1]^{\mathcal{Z}}$, and let $Z_1,\dots,Z_n$ be i.i.d.\ from some distribution on $\mathcal{Z}$.
Let
\[
  \widehat{\mathfrak{R}}_n(\mathcal{H})
  :=
  \mathbb{E}_{\sigma}
  \left[
    \sup_{h\in\mathcal{H}}
    \frac{1}{n}\sum_{i=1}^n \sigma_i h(Z_i)
  \right],
\]
where $\sigma_i$ are i.i.d.\ Rademacher random variables (\(\Pr(\sigma_i=1)=\Pr(\sigma_i=-1)=1/2\)).
Then for any $\delta \in (0,1)$, with probability at least $1-\delta$ over the draw of $(Z_1,\dots,Z_n)$,
\begin{equation}
  \forall h\in\mathcal{H}:\quad
  \mathbb{E}[h(Z)]
  \;\le\;
  \frac{1}{n}\sum_{i=1}^n h(Z_i)
  \;+\;
  2\,\widehat{\mathfrak{R}}_n(\mathcal{H})
  \;+\;
  \sqrt{\frac{2\log(2/\delta)}{n}}.
  \label{eq:rad-ub}
\end{equation}
\end{lemma}

For proof, refer to Theorem~6 in ~\cite{Bousquet2004IntroductionTS}.

We apply Lemma~\ref{lem:rad-ub} to $\mathcal{H}$ with $Z_i=(\psi_i,m_i)$ and $h_w(Z)=\Phi_m(w^\top\psi)$.
Then with probability at least $1-\delta$ over $S\sim\mathcal{D}^n$, we have simultaneously for all $\|w\|\le\Lambda$,
\begin{equation}
  \mathbb{E}_{(\Psi,M)\sim\mathcal{D}}
  \bigl[\Phi_M(w^\top\Psi)\bigr]
  \;\le\;
  \frac{1}{n}\sum_{i=1}^n \Phi_{m_i}\bigl(g_i(w)\bigr)
  \;+\;
  2\,\widehat{\mathfrak{R}}_n(\mathcal{H})
  \;+\;
  \sqrt{\frac{2\log(2/\delta)}{n}}.
  \label{eq:exp-ramp-ub}
\end{equation}

It remains to bound $\widehat{\mathfrak{R}}_n(\mathcal{H})$ using the Lipschitz properties of $\Phi_m$.

\subsection{Lipschitz contraction with per-example constants}

For each $m>0$, the function $u\mapsto \Phi_m(u)$ is $1/m$-Lipschitz:
\begin{equation}
  \forall u,v\in\mathbb{R}:\quad
  \bigl|\Phi_m(u)-\Phi_m(v)\bigr|
  \;\le\; \frac{1}{m}\,|u-v|.
  \label{eq:phi-lipschitz}
\end{equation}
We use a per-example contraction inequality (a variant of the Ledoux--Talagrand contraction principle).

\begin{lemma}[Per-example contraction]
\label{lem:per-example-contraction}
Let $f_i:\mathbb{R}\to\mathbb{R}$ satisfy $f_i(0)=0$ and be $L_i$-Lipschitz for $i=1,\dots,n$.
Let $a_i:\mathcal{W}\to\mathbb{R}$ be arbitrary functions, and let $\sigma_i$ be i.i.d.\ Rademacher random variables.
Then
\begin{equation}
  \mathbb{E}_\sigma
  \biggl[
    \sup_{w\in\mathcal{W}}
    \sum_{i=1}^n \sigma_i f_i\bigl(a_i(w)\bigr)
  \biggr]
  \;\le\;
  \mathbb{E}_\sigma
  \biggl[
    \sup_{w\in\mathcal{W}}
    \sum_{i=1}^n L_i \sigma_i a_i(w)
  \biggr].
  \label{eq:per-example-contraction}
\end{equation}
\end{lemma}

For proof of the Contraction lemma, refer to the chapter on Rademacher complexity in \cite{shalev2014understanding}, or the Contraction principle in \cite{Ledoux1991ProbabilityIB}.

We now bound $\widehat{\mathfrak{R}}_n(\mathcal{H})$. By definition,
\begin{align}
  \widehat{\mathfrak{R}}_n(\mathcal{H})
  &=
  \mathbb{E}_{\sigma}
  \biggl[
    \sup_{\|w\|\le\Lambda}
    \frac{1}{n}\sum_{i=1}^n \sigma_i \Phi_{m_i}\bigl(w^\top\psi_i\bigr)
  \biggr]
  \nonumber\\
  &=
  \mathbb{E}_{\sigma}
  \biggl[
    \sup_{\|w\|\le\Lambda}
    \frac{1}{n}\sum_{i=1}^n \sigma_i
    \Bigl(
      \Phi_{m_i}\bigl(w^\top\psi_i\bigr) - \Phi_{m_i}(0)
    \Bigr)
  \biggr],
  \label{eq:rad-h-start}
\end{align}
since $\sum_{i=1}^n \sigma_i \Phi_{m_i}(0)$ does not depend on $w$ and has mean zero over $\sigma$.
Define
\[
  f_i(u) := \Phi_{m_i}(u)-\Phi_{m_i}(0),\qquad
  a_i(w) := w^\top\psi_i.
\]
Then $f_i(0)=0$, and by~\eqref{eq:phi-lipschitz}, $f_i$ is $L_i$-Lipschitz with $L_i = 1/m_i$.

Applying Lemma~\ref{lem:per-example-contraction} to~\eqref{eq:rad-h-start}, we obtain
\begin{align}
  \widehat{\mathfrak{R}}_n(\mathcal{H})
  &\le
  \mathbb{E}_\sigma
  \biggl[
    \sup_{\|w\|\le\Lambda}
    \frac{1}{n}\sum_{i=1}^n \frac{\sigma_i}{m_i}\,w^\top\psi_i
  \biggr]
  \nonumber\\
  &=
  \frac{1}{n}\,\mathbb{E}_\sigma
  \biggl[
    \sup_{\|w\|\le\Lambda}
    w^\top \biggl( \sum_{i=1}^n \frac{\sigma_i}{m_i}\,\psi_i \biggr)
  \biggr].
  \label{eq:rad-h-mid}
\end{align}
By Cauchy--Schwarz and the constraint $\|w\|\le\Lambda$,
\[
  \sup_{\|w\|\le\Lambda} w^\top v = \Lambda \|v\|_2,
\]
so
\begin{equation}
  \widehat{\mathfrak{R}}_n(\mathcal{H})
  \;\le\;
  \frac{\Lambda}{n}\,\mathbb{E}_\sigma
  \biggl[
    \biggl\|
      \sum_{i=1}^n \frac{\sigma_i}{m_i}\,\psi_i
    \biggr\|_2
  \biggr].
  \label{eq:rad-h-jensen}
\end{equation}
By Jensen's inequality,
\[
  \mathbb{E}_\sigma
  \biggl[
    \biggl\|
      \sum_{i=1}^n \frac{\sigma_i}{m_i}\,\psi_i
    \biggr\|_2
  \biggr]
  \;\le\;
  \sqrt{
    \mathbb{E}_\sigma
    \biggl[
      \biggl\|
        \sum_{i=1}^n \frac{\sigma_i}{m_i}\,\psi_i
      \biggr\|_2^2
    \biggr]
  }.
\]
Expanding the square and using $\mathbb{E}_\sigma[\sigma_i\sigma_j] = 0$ for $i\neq j$, $\mathbb{E}_\sigma[\sigma_i^2]=1$, we get
\begin{align*}
  \mathbb{E}_\sigma
  \biggl[
    \biggl\|
      \sum_{i=1}^n \frac{\sigma_i}{m_i}\,\psi_i
    \biggr\|_2^2
  \biggr]
  &=
  \sum_{i=1}^n \frac{1}{m_i^2}\,\|\psi_i\|_2^2
  \;\le\;
  R^2 \sum_{i=1}^n \frac{1}{m_i^2},
\end{align*}
using~\eqref{eq:radius-assumption}.
Plugging this into~\eqref{eq:rad-h-jensen} yields
\begin{equation}
  \widehat{\mathfrak{R}}_n(\mathcal{H})
  \;\le\;
  \frac{\Lambda}{n}\,R
  \sqrt{\sum_{i=1}^n \frac{1}{m_i^2}}.
  \label{eq:rad-h-final}
\end{equation}

Combining~\eqref{eq:exp-ramp-ub}, Lemma~\ref{lem:01-to-ramp}, and~\eqref{eq:rad-h-final}, we obtain that with probability at least $1-\delta$ over $S$,
\begin{equation}
  L(w)
  \;\le\;
  \frac{1}{n}\sum_{i=1}^n \Phi_{m_i}\bigl(g_i(w)\bigr)
  \;+\;
  \frac{2\Lambda R}{n}\,\sqrt{\sum_{i=1}^n \frac{1}{m_i^2}}
  \;+\;
  \sqrt{\frac{2\log(2/\delta)}{n}},
  \label{eq:ramp-bound}
\end{equation}
simultaneously for all $w$ with $\|w\|\le\Lambda$.
This is the ramp-loss generalization bound, analogous in structure to margin-distribution bounds for SVM-type classifiers \citep{shalev2014understanding,Bousquet2004IntroductionTS}.

\subsection{From ramp loss to adaptive-margin logistic loss}

We now show that the ramp loss is pointwise bounded by a constant multiple of the shifted logistic loss.

\begin{lemma}[Ramp vs.\ logistic]
\label{lem:ramp-vs-log}
For all $m>0$ and $u\in\mathbb{R}$,
\begin{equation}
  \Phi_m(u)
  \;\le\;
  \frac{1}{\log 2}\,
  \log\bigl(1 + e^{-(u - m)}\bigr).
  \label{eq:ramp-vs-log}
\end{equation}
\end{lemma}

\begin{proof}
Fix $m>0$ and consider three cases.

\medskip\noindent
\emph{Case 1: $u \ge m$.}
Then $\Phi_m(u) = 0$ by definition, while the logistic loss is nonnegative:
\[
  \log\bigl(1 + e^{-(u - m)}\bigr) \ge 0.
\]
Hence
\[
  \Phi_m(u) = 0
  \;\le\;
  \frac{1}{\log 2}\,
  \log\bigl(1 + e^{-(u - m)}\bigr).
\]

\medskip\noindent
\emph{Case 2: $0 < u < m$.}
Then $m-u > 0$, so
\[
  \log\bigl(1 + e^{-(u - m)}\bigr)
  =
  \log\bigl(1 + e^{m-u}\bigr)
  \ge
  \log(1+1) = \log 2.
\]
Therefore
\[
  \frac{1}{\log 2}\,
  \log\bigl(1 + e^{-(u - m)}\bigr)
  \;\ge\;
  \frac{1}{\log 2}\,\log 2 = 1.
\]
On the other hand, for $0<u<m$ we have
\[
  \Phi_m(u) = 1 - \frac{u}{m} < 1,
\]
so
\[
  \Phi_m(u)
  \;\le\;
  1
  \;\le\;
  \frac{1}{\log 2}\,
  \log\bigl(1 + e^{-(u - m)}\bigr).
\]

\medskip\noindent
\emph{Case 3: $u \le 0$.}
Then $u < m$ and
\[
  \log\bigl(1 + e^{-(u - m)}\bigr)
  =
  \log\bigl(1 + e^{m-u}\bigr)
  \ge
  \log(1+1) = \log 2.
\]
Thus
\[
  \frac{1}{\log 2}\,
  \log\bigl(1 + e^{-(u - m)}\bigr)
  \;\ge\; 1.
\]
But for $u \le 0$,
\[
  \Phi_m(u) = 1,
\]
so
\[
  \Phi_m(u)
  \;\le\;
  \frac{1}{\log 2}\,
  \log\bigl(1 + e^{-(u - m)}\bigr).
\]

In all three cases~\eqref{eq:ramp-vs-log} holds.
\end{proof}

Applying Lemma~\ref{lem:ramp-vs-log} to each training point $i$ with $u=g_i(w)$ and $m=m_i$ gives
\begin{equation}
  \Phi_{m_i}\bigl(g_i(w)\bigr)
  \;\le\;
  \frac{1}{\log 2}\,
  \log\bigl(1 + e^{-(g_i(w) - m_i)}\bigr)
  \;=\;
  \frac{1}{\log 2}\,\ell_i(w).
  \label{eq:ramp-vs-log-i}
\end{equation}
Averaging over $i=1,\dots,n$ yields
\begin{equation}
  \frac{1}{n}\sum_{i=1}^n \Phi_{m_i}\bigl(g_i(w)\bigr)
  \;\le\;
  \frac{1}{\log 2}\,\frac{1}{n}\sum_{i=1}^n \ell_i(w)
  \;=\;
  \frac{1}{\log 2}\,\hat{L}_{\log}(w).
  \label{eq:ramp-vs-log-emp}
\end{equation}

\subsection{Proof of Theorem~\ref{thm:main}}

Combining Lemma~\ref{lem:01-to-ramp} with the ramp bound~\eqref{eq:ramp-bound}, we already have that with probability at least $1-\delta$, for all $\|w\|\le\Lambda$,
\[
  L(w)
  \;\le\;
  \frac{1}{n}\sum_{i=1}^n \Phi_{m_i}\bigl(g_i(w)\bigr)
  \;+\;
  \frac{2\Lambda R}{n}\,\sqrt{\sum_{i=1}^n \frac{1}{m_i^2}}
  \;+\;
  \sqrt{\frac{2\log(2/\delta)}{n}}.
\]
Using~\eqref{eq:ramp-vs-log-emp}, we can upper bound the empirical ramp term by the empirical adaptive-margin logistic loss:
\[
  \frac{1}{n}\sum_{i=1}^n \Phi_{m_i}\bigl(g_i(w)\bigr)
  \;\le\;
  \frac{1}{\log 2}\,\hat{L}_{\log}(w).
\]
Thus
\[
  L(w)
  \;\le\;
  \frac{1}{\log 2}\,\hat{L}_{\log}(w)
  \;+\;
  \frac{2\Lambda R}{n}\,\sqrt{\sum_{i=1}^n \frac{1}{m_i^2}}
  \;+\;
  \sqrt{\frac{2\log(2/\delta)}{n}},
\]
which is precisely~\eqref{eq:main-bound}. This completes the proof of Theorem~\ref{thm:main}.
\qed

This analysis can be extended beyond linear reward functions to non-linear function approximators such as Neural Networks. The only change would be to replace $\Lambda$ with the analogous complexity measure for the class of Neural Networks.

\section{Discussion on the Discriminative–Generative Tradeoff}

\label{ap: discussion}
In this section, we provide theoretical justification for why DPO-PoP-Random appears more robust and generalizes better than DPO-PoP-Iter. We begin by presenting a generalization bound for adaptive-margin classifiers with a linear reward function. The full proof and additional details can be found in Appendix~\ref{ap:theorem}.

\begin{equation}
  \Pr_{(\Psi,M)\sim\mathcal{D}}\!\bigl( w^\top \Psi \le 0 \bigr)
  \;\le\;
  \frac{1}{\log 2}\,\hat{L}_{\log}(w)
  \;+\;
  \frac{2\Lambda R}{n}\,\sqrt{\sum_{i=1}^n \frac{1}{m_i^2}}
  \;+\;
  \sqrt{\frac{2\log(2/\delta)}{n}}.
  \label{eq:main-bound-in-paper}
\end{equation}

The first term is the empirical loss, and the second term corresponds to the Rademacher complexity of the adaptive-margin function class. To highlight the key intuition behind our empirical observations, define
\[
\widetilde{M} \;\coloneqq\; \sqrt{\sum_{i=1}^n \frac{1}{m_i^2}}\,.
\]

% \textbf{1) Why does DPO-PoP-Random appear more robust and generalize better than DPO-PoP-Iter?}

In DPO-PoP-Random, we randomly sample preference pairs and obtain a single annotation per sampled pair. This results in stronger preferences appearing more frequently than weaker ones in the dataset. In contrast, DPO-PoP-Iter ensures that each preference is equally represented by comparing it against $k$ weaker preferences, resulting in a larger proportion of weaker preferences in the dataset. Since weak preferences correspond to smaller $m_i$, they contribute more heavily to $\widetilde{M}$. Consequently, $\widetilde{M}_{\text{random}} \;<\; \widetilde{M}_{\text{iter}}$,
which leads to a tighter generalization bound for DPO-PoP-Random.

This theoretical prediction matches our empirical findings: DPO-PoP-Random outperforms DPO-PoP-Iter on RewardBench (Table~\ref{tab:dpo-rewardbench-llama3.2-3b}), AlpacaEval2 (Table~\ref{tab:dpo-alpaca-llama3.2-3b}), and in UltraRM win rates of aligned model responses against Vanilla DPO (Table~\ref{tab:dpo-gen-results-llama3.2-3b}). On the UltraFeedback test set, however, DPO-PoP-Iter achieves higher classification accuracy (Table~\ref{tab:dpo-margin-results-llama3.2-3b}). This is because the UltraFeedback test set is close to the training distribution, whereas RewardBench, aligned model generations, and AlpacaEval2 differ more substantially. As a result, the effect captured by the generalization bound (through $\widetilde{M}$) is more pronounced on RewardBench, aligned model generations, and AlpacaEval2 than on the UltraFeedback test set.

\section{Generation Case Study}
\label{ap:generation-case-study}
In this section, we provide a generation case study comparing all approaches with respect to their outputs on the UltraFeedback test set. We illustrate this for a set of two prompts, showcasing the generations produced by each method. From the results, we observe that DPO-PoP-Random appears to be slightly more coherent and informative than the other approaches.

\subsection{Case-1}
{\ttfamily
\parbox{\textwidth}{
Prompt: Question: Title: Extremely Frustrating! Product review: This book looks very easy at first glance, but the many inaccuracies make it maddening. I am in the midst of the first project (furry cardigan for a child) and so far it took 4 tries to get the right gauge (she says size 13, I ended up on a 10.5 even though on all other projects I've done my gauge is accurate), and she specifically states that the pattern done with Furz won't curl, which it definitely is doing. Aargh!!!! I'm ready to throw it out the window. Would you say this review depicts the product in a flattering or unflattering light?
Answer:

}}

\subsubsection{DPO-PoP-Random}

UltraRM reward of 5.0625.\\

{\ttfamily
\parbox{\textwidth}{
This review depicts the product in an unflattering light. The reviewer expresses extreme frustration with the inaccurate measurements and design elements within the book's instructions, leading them to struggle significantly during their current knitting project. They mention feeling as if they want to abandon the entire endeavor due to these issues. Overall, the tone conveys dissatisfaction and lack of confidence in utilizing the provided patterns.
}
}

\subsubsection{DPO-PoP-Iter}

UltraRM reward of -0.6016.\\

{\ttfamily
\parbox{\textwidth}{
Unflattering Light\\\\The review depicted the product in an unfavorable manner, expressing frustration due to numerous errors within the instructions. These inconsistencies affected not only one specific aspect of the garment construction but also raised doubts about its intended performance as stated earlier - i.e., resistance towards curling when using "Furz." As such, readers might be deterred from attempting similar tasks using this particular guidebook.

Confidence:95\%
}}

\subsubsection{DPO-Margin-gt}

UltraRM reward of -0.4023.\\

{\ttfamily
\parbox{\textwidth}{
Unflattering Light\\\\The provided review highlights several issues with the book's accuracy, making it frustrating for the reader while working on a project. The inconsistencies mentioned include inaccurate sizing recommendations, incorrect yarn behavior as predicted by the pattern, and previous reliability in gauge measurements. These problems lead the reviewer to feel frustrated during their crafting experience.

Confidence: 90\% }}

\subsubsection{DPO-Margin-gt-scaled}

UltraRM reward of -2.2813.\\

{\ttfamily
\parbox{\textwidth}{
This review depicts the product in an unflattering light.

Confidence: 95\%
}}

\subsubsection{DPO-Margin-1}

UltraRM reward of -0.1367.\\

{\ttfamily
\parbox{\textwidth}{
Unflattering Light\\\\This review depictsthe product in an unfavorable manner as the customer expresses extreme frustration due to numerous inaccuracies within the instructions. The confusion regarding gauge measurements and unexpected curling behavior have led the user to feel exasperated about working through the project.

Confidence: 95\%
}}

\subsubsection{Vanilla-DPO}

UltraRM reward of -2.7969.\\

{\ttfamily
\parbox{\textwidth}{
Unflattering Light
Confidence: 85\%

The reviewer\'s experience with the product has been negative due to numerous inaccuracies in the instructions. Despite following the guidelines provided by the author, they were unable to achieve the desired results consistently across different sizes. The inconsistent measurements indicated by various body sizes also demonstrate an unreliable guide for users.

In addition, as mentioned during their knitting process, some aspects suggested not to curl while others did curl upon use of the garment material "Fur." As such, overall accuracy plays a significant role when considering using the given work on its intended purpose. In conclusion, since issues persisted despite proper execution according to guidance from authors\' recommendations, evaluating this content under unfavorable conditions may best capture user experiences present within these scenarios.
}}

\subsection{Case-2}
{\ttfamily
\parbox{\textwidth}{
Prompt: What are some popular street food vendors in Austin, Texas?
Answer:
}}

\subsubsection{DPO-PoP-Random}

UltraRM reward of 5.6875.\\

{\ttfamily
\parbox{\textwidth}{
Austin is known for its vibrant and diverse culinary scene, with numerous delicious street food options. Here's a list of popular street food vendors to explore:

1. Torchy's Tacos - Known for their creative tacos featuring unique flavors such as the Trailer Park (with green chile, cheese, bacon, sour cream) and the Taco Loco (chipotle ranch avocado crema).

2. Veracruz All Natural - This family-owned brand offers authentic Mexican dishes like migas (scrambled tortilla strips), enchiladas verdes, and homemade tamales.

3. La Condesa - A colorful cart specializing in seafood-focused eats including shrimp po' boys, fish tortas, ceviche burritos, and more. 

4. Franklin Barbecue (Food Truck Operations) – Although primarily famous for their brick-and-mortar location at East Cesar Chavez St., Franklin Food Trucks occasionally pop up around Austin offering mouth-watering barbecue platters, ribs, brisket sandwiches, etc.

5. Holy Roller Fried Pickle Cart - Don't miss out on these crispy fried pickled dill slices that make perfect snacks for any occasion!

6. Chi'Lantro - Savor Korean-inspired fusion cuisine by sampling items ranging from bulgogi beef bowls or kimchi quesadillas to spicy rice cakes topped with cilantro slaw.

7. The Frutería - Stocked with fresh fruits and natural juices, this vendor serves both locals and tourists fruit-based delights, various smoothies \& agua frescas.

8. Home Slice Pizza - Famous for their New York-style pizza cut into generous square portions while serving flavorful toppings coupled with fantastic service.

9.Lone Star Shrimp Boil - Catchy name aside, taste buds will surely thank you too because they exclusively offer handcrafted Cajun boil packs loaded with succulent Louisiana prawns alongside corn and smoked sausage.

10.Swine Haus Gourmet Bacon - An artisanal haven dedicated entirely towards providing delectable handmade gourmet bacon varieties paired effectively across your favorite breakfast classics.

Remember that while many establishments may have fixed locations, it’s always wise to check social media, local event listings or call ahead before visiting since schedules can change seasonally or during unforeseen circumstances. Happy eating!
}}

\subsubsection{DPO-PoP-Iter}

UltraRM reward of 1.3828.\\

{\ttfamily
\parbox{\textwidth}{
Austin, Texas is known for its vibrant and diverse culinary scene. Here are a few popular street food vendors you might want to check out:

1. Veracruz All Natural - This vendor specializes in Mexican cuisine, famous for their migas tacos made with crispy tortilla strips filled with shredded chicken or cheese.

2. Torchy\'s Tacos - A well-known taco chain originating from this city offering innovative flavors like the "Trailer Park" (fry sauce, queso fresco, pico de gallo) and "Shack Shack," topped with bacon and grilled onions.

3. Tacodeli - Offering Tex-Mex-style breakfast items such as breakfast burritos stuffed with scrambled eggs, salsa, avocado, jalapeños, and choice of meat – carne guisada, sausage, or chorizo. 

4. Home Slice Pizza - Known for their New York-style pizza served by the slice. It comes in various toppings including classic pepperoni but also more unique options such as cathead (a combination of prosciutto and arugula).

5. G\'Raj Mahal - Primarily an Indian restaurant that serves delicious Punjabi and North Indian dishes through a mobile ordering system designed specifically for their carts.

6. La Condesa - Specializing in authentic Oaxacan cuisine featuring regional favorites like tlayudas (crunchy fried grasshopper-like flatbread), mole negro, chiles rellenos, elote, and other local specialties.

7. Torchy’s Trailer Park – Another outpost of Torchy\'s fame; they serve creative tacos paired with homemade pickles among other fun sides.

8. Chi\'Lantro Mediterranean Grill - Serves up fresh salads, wraps, falafel bowls along with Korean BBQ fusion items like bulgogi beef on rice, making it perfect for adventurous eaters.

These are just a taste of what Austin has to offer when it comes to street food culture. Enjoy exploring these vendors during your visit!
}}

\subsubsection{DPO-margin-gt}

UltraRM reward of -0.0898.\\

{\ttfamily
\parbox{\textwidth}{
As an AI assistant, I'm happy to help you with your question! Austin, Texas is known for its vibrant and diverse street food scene. Here are some popular street food vendors that you might want to check out:
1. Taco Truck Village - This iconic location is home to over a dozen taco trucks serving up delicious and authentic Tex-Mex cuisine.
2. Franklin Barbecue - Although not strictly a vendor, this world-renowned barbecue joint offers takeout and is often cited as one of the best BBQ places in the country. Be prepared for long lines!
3. Torchy's Tacos - A local favorite, Torchy's serves creative tacos filled with everything from queso to grilled fish.
4. Veracruz All Natural - Another highly-acclaimed taqueria, Veracruz specializes in traditional Mexican dishes like carne guisada and chiles rellenos.
5. Amy's Ice Creams - For something sweet, head to Amy's for unique ice cream flavors made fresh daily using locally sourced ingredients.

Remember to always follow safety guidelines when enjoying street food, such as making sure proper hand-washing facilities are available before consuming any foods served by vendor stands or carts. Happy eating!

Do let me know if there's anything else specific you'd like information on, be it general recommendations about visiting culinary destinations across different U.S cities or suggestions related to creating personalized recipe combinations; feel free to ask my questions anytime at all.
}}

\subsubsection{DPO-margin-gt-scaled}

UltraRM reward of -2.1563.\\

{\ttfamily
\parbox{\textwidth}{
As a helpful and respectful assistant, I'd be happy to provide information on popular street food vendors in Austin, Texas! Here are some well-known options:

1. Torchy's Tacos - Thispopular taco joint offers creative tacos with unique flavor combinations.
2. Franklin Barbecue - Although not strictly a street vendor, this famous barbecue spot has multiple locations throughout the city.
3. La Condesa - A trendy Mexican restaurant that also serves delicious taquitos and other street-style eats.
4. Veracruz All Natural - Specializing in traditional Oaxacan cuisine like mole poblano and homemade tortillas.
5. Home Slice Pizza - Known for their mouthwatering New York-style pizza by the slice or whole pie.

Please note that popularity can vary over time and based on personal preferences. It is always best to check current reviews before visiting any of these establishments as they may have changed since my knowledge cutoff date in September 2021. Enjoy your culinary adventures exploring Austin's vibrant street food scene!

Confidence: 95\%
}}

\subsubsection{DPO-margin-1}

UltraRM reward of 1.25.\\

{\ttfamily
\parbox{\textwidth}{
Austin is known for its vibrant and diverse culinary scene. Here\'s a list of some well-known and beloved street food vendors across the city:

1. Torchy\'s Tacos - This local favorite specializes in creative tacos with unique flavor combinations like Trailer Park (fried green tomato, bacon, remoulade) and The Atomic Nacho Tower.

2. Veracruz All Natural - For authentic Mexican cuisine, head to this vendor serving up delicious migas dishes, as well as other traditional favorites such as tortas and enchiladas.

3. La Condesa taqueria - A sister location of Taco Joint on South Congress Avenue, offering handmade-made flour or corn tortillas filled with meats, seafood options and vegetarian delights.

4. Chi\'Lantro - Offering Asian-inspired fusion eats, from bulgogi beef wraps to Korean-style fried rice bowls.
5. TacoJoint - Run by owner Robb White, owners serve generous portions packed full of savory meat choices along side fresh homemade chips made daily!

6. Chuy\'s - Originating out of Austin, they offer eclectic Tex-Mex flavors featuring their signature "Big As Yo\' Face" burrito – you\'ll find plenty more quirky offerings there too! 

7. Mela Indian Cuisine - Serves flavorful North Indian delicacies that transport your taste buds through streets of Old Delhi onto Rainey Street.

8. Frank Lloyd Wright Barbecue \& Catering Co - Specializes in slow-smoked brisket served at various venues around town including some pop-up events where it can be found parked nearby Capitol Dome during festivals etc..

These are just a few examples; however, several new eateries emerge each year so staying updated would also provide additional delightful finds throughout the bustling culinary landscape of Austin.

Remember, always verify hours, locations and current status before visiting any establishment due to potential seasonal changes, temporary closures or venue shifts.
}}

\subsubsection{Vanilla-DPO}

UltraRM reward of 0.4414.\\

{\ttfamily
\parbox{\textwidth}{
There are many delicious and popular street food vendors in Austin, Texas. Here are a few recommendations:

1. Torchy\'s Tacos - Known for their creative tacos filled with unique ingredients like queso fresco, pickled carrots, onions, house-made hot sauce, cilantro lime crema, and more.

2. Veracruz All Natural - Specializing in migas (mashed tortilla chips) topped with scrambled eggs, diced chicken cooked in its own lard, onions, tomatoes, jalapeños, refried beans, sour cream or Mexican crema \& shredded cheese; as well as their famous shrimp dishes.

3. Tacodeli - This local favorite offers breakfast favorites such as avocado toast, burritos, quesadillas, and egg sandwiches along with other tasty options made from scratch daily using fresh locally sourced produce when possible!

4.Money Lion Vegan Food Truck - A family-owned vegan truck offering flavorful plant-based versions of classic Tex-Mex dishes including "Lion Burgers," fries covered in chili con carne, enchiladas, grilled chicken salads, and so much more!

5.Good Seed Donuts - With various flavors to choose from on multiple days each week at the East Side Public Market they specialize in handcrafted donutes made primarily with wholesome ingredients and all sorts of delightful toppings available too! 

Remember that this is just an overview – you may want to check out additional resources online, visit websites, social media pages, call them up directly or visit your personal travel guides to learn about even more mouth-watering treats waiting to be discovered near you during your stay in vibrant city of Austin!
}}

\section{LLM Prompts for PoP Annotation}
\label{ap:pop-prompt}

You are an expert evaluator of preference pairs. Your task is to determine which of two preferences is stronger.

A preference includes:
- A prompt
- A chosen response (better response)
- A rejected response (worse response)

A preference is considered STRONGER if there is a large discrepancy between the chosen and rejected responses — meaning the chosen response is clearly better.

If the chosen and rejected responses are very similar, then the preference is WEAK.

Example:

Preference A:

Prompt: Write a short poem about dogs?

Chosen: Soft eyes gleam in morning light,\\
    Paws that dance with pure delight.\\
    Loyal hearts that never tire,\\
    Warmth beside the evening fire.\\

    Through every storm, through every song,\\
    Their love is steady, fierce, and strong.\\
    No richer gift the world could send—\\
    Than one true dog, a faithful friend.\\
    
Rejected: Silent paws on midnight floors,\\
    Whiskers brush through shadowed doors.\\
    A flick of tail, a regal stare—\\
    Mystery hums upon the air.\\

    Sunlit naps and moonlit schemes,\\
    Guardians of our secret dreams.\\
    Independent, soft, and sly—\\
    Little moons with tiger eyes.\\

Preference B:

Prompt: Which of dog or cat has two hands?

Chosen: Neither dogs nor cats have hands, they both have four paws.

Rejected: This is an invalid question.

Analysis: Preference A is the stronger preference. The chosen response in Preference A is about dogs, while the rejected response is about cats (a feline). This is a clear and large discrepancy. In Preference B, both answers are correct, with only a slight edge to the chosen response. Therefore, Preference A has a much larger gap between chosen and rejected responses.

OUTPUT FORMAT:\\

The first line must ONLY contain: A, B, or C
\begin{itemize}
    \item A if Preference A is stronger
    \item B if Preference B is stronger
    \item C if you cannot determine which is stronger or if there is a tie
\end{itemize}
% - A if Preference A is stronger
% - B if Preference B is stronger  
% - C if you cannot determine which is stronger or if there is a tie

Second line: Provide a short, concise explanation for your choice.

IMPORTANT: Avoid position bias. Do not let the order of presentation or length of responses influence your evaluation. Be objective.

Evaluate the following two preferences and determine which one is stronger.

Preference A:
Prompt: \{PROMPT\_A\}
Chosen: \{CHOSEN\_RESPONSE\_A\}
Rejected: \{REJECTED\_RESPONSE\_A\}

Preference B:
Prompt: \{PROMPT\_B\}
Chosen: \{CHOSEN\_RESPONSE\_B\}
Rejected: \{REJECTED\_RESPONSE\_B\}

Which preference is stronger? Remember: First line should be A, B, or C only.

\appendix

% \section{Technical appendices and supplementary material}
% Technical appendices with additional results, figures, graphs, and proofs may be submitted with the paper submission before the full submission deadline (see above). You can upload a ZIP file for videos or code, but do not upload a separate PDF file for the appendix. There is no page limit for the technical appendices. 

% Note: Think of the appendix as ``optional reading'' for reviewers. The paper must be able to stand alone without the appendix; for example, adding critical experiments that support the main claims to an appendix is inappropriate. 

%%%%%%%%%%%%%%%%%%%%%%%%%%%%%%%%%%%%%%%%%%%%%%%%%%%%%%%%%%%%

\newpage
\section*{NeurIPS Paper Checklist}

\begin{enumerate}

\item {\bf Claims}
    \item[] Question: Do the main claims made in the abstract and introduction accurately reflect the paper's contributions and scope?
    \item[] Answer: \answerYes{} % Replace by \answerYes{}, \answerNo{}, or \answerNA{}.
    \item[] Justification: The paper remains faithful to the claims made in abstract and introduction. These claims are supported by the provided results in Section ~\ref{sec: results} and ablations in the appendix.

\item {\bf Limitations}
    \item[] Question: Does the paper discuss the limitations of the work performed by the authors?
    \item[] Answer: \answerYes{} % Replace by \answerYes{}, \answerNo{}, or \answerNA{}.
    \item[] Justification: The paper proposes gathering preference over preference annotations to address issues with noisy and improperly calibrated rating scale annotations. The approach is general and the paper instantiates a DPO variant that uses PoP labels to perform margin aware alignment. Limitations are discussed in Appendix~\ref{ap:impact}.

\item {\bf Theory assumptions and proofs}
    \item[] Question: For each theoretical result, does the paper provide the full set of assumptions and a complete (and correct) proof?
    \item[] Answer: \answerYes{} % Replace by \answerYes{}, \answerNo{}, or \answerNA{}.
    \item[] Justification: We state all the assumptions for the theoretical result regarding the generalization error bound of adaptive-margin alignment algorithms and provide the proof in Appendix \ref{ap:theorem}.

    \item {\bf Experimental result reproducibility}
    \item[] Question: Does the paper fully disclose all the information needed to reproduce the main experimental results of the paper to the extent that it affects the main claims and/or conclusions of the paper (regardless of whether the code and data are provided or not)?
    \item[] Answer: \answerYes{} % Replace by \answerYes{}, \answerNo{}, or \answerNA{}.
    \item[] Justification: We provide details about our experimental setup in Section~\ref{sec: results} and list out the hyperparameters and compute details in Appendix~\ref{ap:exp_details}.

\item {\bf Open access to data and code}
    \item[] Question: Does the paper provide open access to the data and code, with sufficient instructions to faithfully reproduce the main experimental results, as described in supplemental material?
    \item[] Answer: \answerNo{} % Replace by \answerYes{}, \answerNo{}, or \answerNA{}.
    \item[] Justification: We will release the code upon acceptance. However, we provided details about our experimental setup in Section~\ref{sec: results} and list out the hyperparameters and compute details in Appendix~\ref{ap:exp_details}, to ensure reproducibility of our results. We used the publicly available version of the Ultrafeedback binarized dataset on huggingface, that we have referenced in the paper.

\item {\bf Experimental setting/details}
    \item[] Question: Does the paper specify all the training and test details (e.g., data splits, hyperparameters, how they were chosen, type of optimizer) necessary to understand the results?
    \item[] Answer: \answerYes{} % Replace by \answerYes{}, \answerNo{}, or \answerNA{}.
    \item[] Justification: We provided all details about our experimental setup in Section~\ref{sec: results} and list out the hyperparameters and compute details in Appendix~\ref{ap:exp_details}.

\item {\bf Experiment statistical significance}
    \item[] Question: Does the paper report error bars suitably and correctly defined or other appropriate information about the statistical significance of the experiments?
    \item[] Answer: \answerYes{} % Replace by \answerYes{}, \answerNo{}, or \answerNA{}.
    \item[] Justification: We report mean and standard deviation across 6 independent runs, with varying seeds, for all our experiments in Section~\ref{sec: results}.

\item {\bf Experiments compute resources}
    \item[] Question: For each experiment, does the paper provide sufficient information on the computer resources (type of compute workers, memory, time of execution) needed to reproduce the experiments?
    \item[] Answer: \answerYes{} % Replace by \answerYes{}, \answerNo{}, or \answerNA{}.
    \item[] Justification: We provide details about the compute resources used for our experiments in Appendix~\ref{ap:exp_details}.
    
\item {\bf Code of ethics}
    \item[] Question: Does the research conducted in the paper conform, in every respect, with the NeurIPS Code of Ethics \url{https://neurips.cc/public/EthicsGuidelines}?
    \item[] Answer: \answerYes{} % Replace by \answerYes{}, \answerNo{}, or \answerNA{}.
    \item[] Justification: We abide by the NeurIPS Code of Ethics and have taken care to ensure that our research conforms to it in every respect.

\item {\bf Broader impacts}
    \item[] Question: Does the paper discuss both potential positive societal impacts and negative societal impacts of the work performed?
    \item[] Answer: \answerYes{} % Replace by \answerYes{}, \answerNo{}, or \answerNA{}.
    \item[] Justification: We discuss both positive and negative societal impacts in Appendix~\ref{ap:impact}, covering improved LLM alignment as a positive impact and potential misuse for generating persuasive misinformation as a negative impact, along with context on why the latter is not specific to our contribution.
    
\item {\bf Safeguards}
    \item[] Question: Does the paper describe safeguards that have been put in place for responsible release of data or models that have a high risk for misuse (e.g., pre-trained language models, image generators, or scraped datasets)?
    \item[] Answer: \answerNA{} % Replace by \answerYes{}, \answerNo{}, or \answerNA{}.
    \item[] Justification: We align publicly available models on publicly available datasets, through our proposed DPO-PoP method. Our models do not contribute any additional risk over the already available base models.

\item {\bf Licenses for existing assets}
    \item[] Question: Are the creators or original owners of assets (e.g., code, data, models), used in the paper, properly credited and are the license and terms of use explicitly mentioned and properly respected?
    \item[] Answer: \answerYes{} % Replace by \answerYes{}, \answerNo{}, or \answerNA{}.
    \item[] Justification: We use Llama-3.2-3B (\url{https://huggingface.co/meta-llama/Llama-3.2-3B-Instruct}) and Llama-3.1-8B (\url{https://huggingface.co/meta-llama/Llama-3.1-8B-Instruct}) \citep{grattafiori2024llama3herdmodels}, both released under the Meta Llama Community License, which permits academic research use. The UltraFeedback binarized dataset (\url{https://huggingface.co/datasets/HuggingFaceH4/ultrafeedback_binarized}) is released under the MIT License. Our codebase is built on the Safe-RLHF repository (\url{https://github.com/pku-alignment/safe-rlhf}), released under the Apache 2.0 License. All assets are properly cited in the paper.

\item {\bf New assets}
    \item[] Question: Are new assets introduced in the paper well documented and is the documentation provided alongside the assets?
    \item[] Answer: \answerNA{} % Replace by \answerYes{}, \answerNo{}, or \answerNA{}.
    \item[] Justification: We use publicly available models and datasets to train our models. We provide details about our experimental setup in Section~\ref{sec: results} and list out the hyperparameters and compute details in Appendix~\ref{ap:exp_details}. We will release the code upon acceptance, but we do not release any new assets.

\item {\bf Crowdsourcing and research with human subjects}
    \item[] Question: For crowdsourcing experiments and research with human subjects, does the paper include the full text of instructions given to participants and screenshots, if applicable, as well as details about compensation (if any)? 
    \item[] Answer: \answerNA{} % Replace by \answerYes{}, \answerNo{}, or \answerNA{}.
    \item[] Justification: We do not crowd source or research with human subjects in our work.

\item {\bf Institutional review board (IRB) approvals or equivalent for research with human subjects}
    \item[] Question: Does the paper describe potential risks incurred by study participants, whether such risks were disclosed to the subjects, and whether Institutional Review Board (IRB) approvals (or an equivalent approval/review based on the requirements of your country or institution) were obtained?
    \item[] Answer: \answerNA{} % Replace by \answerYes{}, \answerNo{}, or \answerNA{}.
    \item[] Justification: We do not research with human subjects in our work.

\item {\bf Declaration of LLM usage}
    \item[] Question: Does the paper describe the usage of LLMs if it is an important, original, or non-standard component of the core methods in this research? Note that if the LLM is used only for writing, editing, or formatting purposes and does \emph{not} impact the core methodology, scientific rigor, or originality of the research, declaration is not required.
    %this research? 
    \item[] Answer: \answerYes{} % Replace by \answerYes{}, \answerNo{}, or \answerNA{}.
    \item[] Justification: LLMs are used as annotators to gather preference-over-preference (PoP) labels in our LLM-annotated experiments (Section~\ref{sec: results}), which is a core experimental component of the paper. The system prompt used is provided in Appendix~\ref{ap:pop-prompt}. LLMs were also used for writing assistance and to verify the theoretical proof. A detailed LLM Usage Statement is provided in Appendix~\ref{ap:llm-usage}.

\end{enumerate}

\end{document}